\documentclass[letterpaper]{article}
\usepackage{aaai2026} 
\usepackage{times} 
\usepackage{helvet} 
\usepackage{courier} 
\usepackage[hyphens]{url} 
\usepackage{graphicx} 
\usepackage{graphicx}  
\usepackage{natbib}  
\usepackage{caption}  
\usepackage{algorithm}
\usepackage{algorithmic}
\usepackage{booktabs}
\usepackage{amsmath} 
\usepackage{amssymb}
\usepackage{graphicx} 
\usepackage{url}
\usepackage{amsmath}
\usepackage{amssymb}
\usepackage{amsthm}

\newtheorem{theorem}{Theorem}
\newtheorem{lemma}{Lemma}

\usepackage{newfloat}
\usepackage{listings}
\DeclareCaptionStyle{ruled}{labelfont=normalfont,labelsep=colon,strut=off} 
\floatstyle{ruled}
\newfloat{listing}{tb}{lst}{}
\floatname{listing}{Listing}
\title{Lightweight Optimal-Transport Harmonization on Edge Devices}
\author{
    Maria Larchenko\textsuperscript{\rm 1},
    Dmitry Guskov\textsuperscript{\rm 2}, 
    Alexander Lobashev\textsuperscript{\rm 2}, 
    Georgy Derevyanko\textsuperscript{\rm 1}
}
\affiliations{
    \textsuperscript{\rm 1} Magicly AI, Dubai, UAE \\
    \textsuperscript{\rm 2}Glam AI, San Francisco, USA\\
\{mariia.larchenko, guskov01dmitry, lobashevalexander, georgy.derevyanko\}@gmail.com
    
}

\begin{document}

\maketitle

\begin{abstract}
Color harmonization adjusts the colors of an inserted object so that it perceptually matches the surrounding image, resulting in a seamless composite. The harmonization problem naturally arises in augmented reality (AR), yet harmonization algorithms are not currently integrated into AR pipelines because real-time solutions are scarce. In this work, we address color harmonization for AR by proposing a lightweight approach that supports on-device inference. For this, we leverage classical optimal transport theory by training a compact encoder to predict the Monge-Kantorovich transport map. We benchmark our MKL-Harmonizer algorithm against state-of-the-art methods and demonstrate that for real composite AR images our method achieves the best aggregated score. We release our dedicated AR dataset of composite images with pixel-accurate masks and data-gathering toolkit to support further data acquisition by researchers.
\end{abstract}

\begin{links}
    \link{Code}{https://github.com/maria-larchenko/mkl-harmonizer}
\end{links}

\section{Introduction}
Composite images are created by pasting a foreground object onto a background image under a binary mask. Although spatial alignment may be perfect, the inserted region usually looks out of place because it was captured under different illumination, camera response, or post-processing pipeline. Image harmonization modifies the pasted region so that its appearance becomes perceptually consistent with its surroundings, yielding a harmonized image. Existing harmonization methods are predominantly offline and assume desktop-class resources, limiting their adoption in interactive applications \cite{niu2021making}.

\begin{figure}
    \centering
    \includegraphics[width=0.375\linewidth]{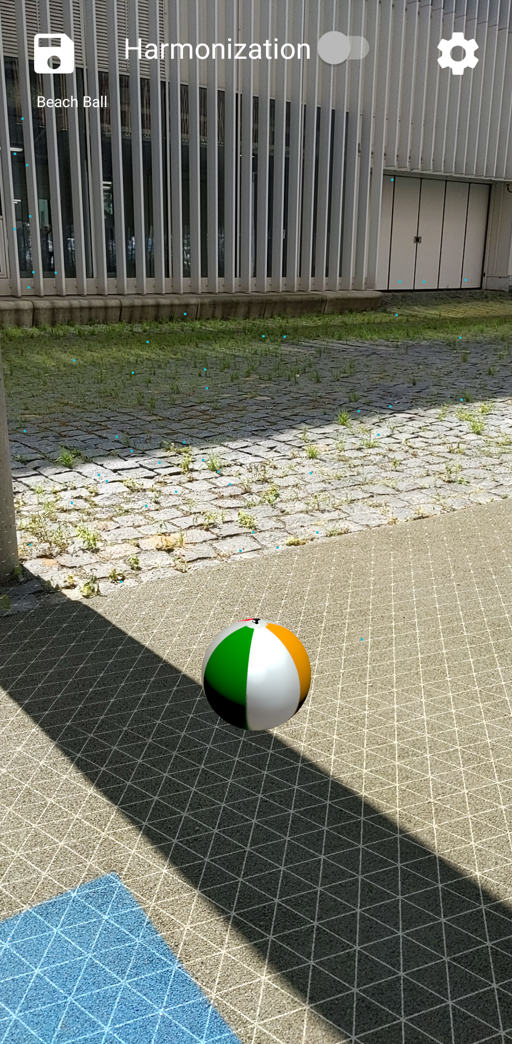}
    \includegraphics[width=0.375\linewidth]{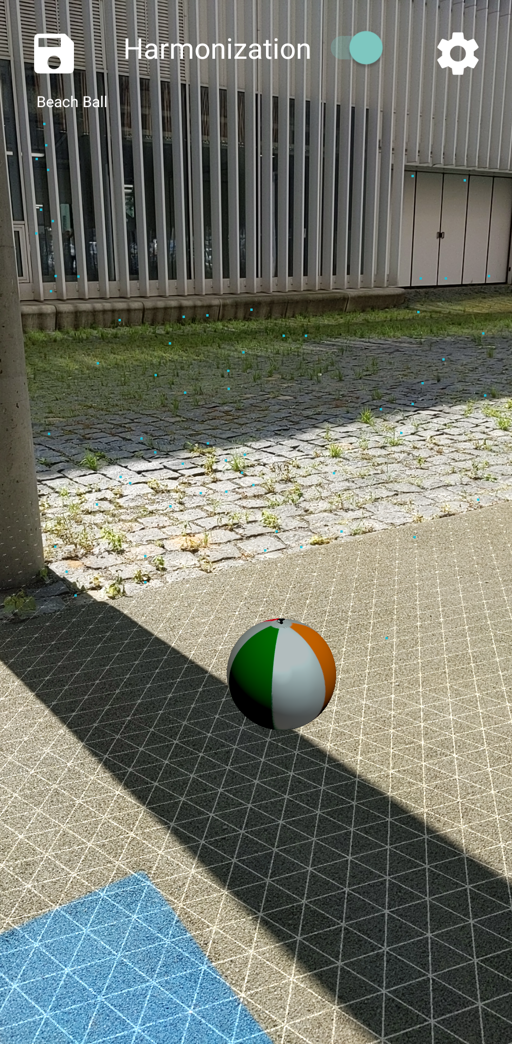}
    \caption{Harmonization running on edge device.}
    \label{fig:pair_of_balls}
\end{figure}

Early and influential works adopted encoder-decoder architectures, such as U-Net \cite{ronneberger2015u}, to perform a dense, pixel-to-pixel mapping of the foreground. Models like DoveNet \cite{cong2020dovenet}, RainNet \cite{ling2021region}, or IntrinsicHarmony \cite{guo2021intrinsic} were able to learn rich semantic representations, leading to significant improvements over classical techniques. However, dense prediction models share a fundamental limitation in their computational and memory requirements. Consequently, research was largely confined to low-resolution inputs, typically 256x256 pixels.

To address these limitations, Harmonizer \cite{ke2022harmonizer} regresses coefficients for filters like brightness and contrast, while DCCF \cite{xue2022dccf} proposes a set of neural color filters whose application is controlled by a predicted coefficient map. 
More recent approaches have focused on learning the transformation itself. PCT-Net \cite{guerreiro2023pct} learns to predict the parameters for a smooth field of pixel-wise affine color transforms. A key contribution was to perform upsampling on the low-resolution parameter map rather than the image itself.
INR-Harmonization (HINet) \cite{chen2023dense} introduced the use of an Implicit Neural Representation, where an encoder predicts the weights of a small MLP that maps a pixel's coordinate to its final harmonized color. 

While deep learning has dominated recent literature, the principles of color transfer are rooted in classical statistical methods. Recently, Optimal Transport (OT) theory has experienced a significant resurgence, particularly in the context of generative modeling \cite{liu2023flow}.

Historically, OT provided a way to map one color distribution to another. For Gaussian approximation of source and target color distributions, optimal transport map has a closed-form linear solution, known as Monge-Kantorovich Linear map (MKL) \cite{pitie2007linear}.
For the color transfer problem, the MKL filter yields perceptually coherent results. 
Later, Rabin et al. \cite{rabin2010regularization} introduced a relaxed formulation of OT for adaptive color transfer and Bonneel et al. \cite{bonneel2013example} applied OT principles to the complex task of video color grading.
MKL filters first introduced in 2007 are still a strong competitor for the more recent color transfer algorithms \cite{larchenko2025color}.

To apply MKL color filter one has to calculate 12 parameters, derived from statistics of source and target color distributions. To the best of our knowledge, this approach was not previously tested for the color harmonization problem. Since the target color distribution is unknown for inserted object, our approach is to train a network to predict the 12 parameters of Monge-Kantorovich Linear transformation. The proposed solution is promisingly lightweight, fast and fits into the constraints of on-device augmented reality.

Our work makes the following contributions
\begin{itemize}
    \item We propose MKL-Harmonizer, a novel lightweight solution for color harmonization, based on prediction of Monge-Kantorovich Linear maps.
    \item We built a dedicated tool to gather AR-specific composite images and masks in the wild and have collected a set of 327 images for human evaluation.
    \item We demonstrate that our approach achieves superior performance in terms of aggregated speed-quality metric.
    \item We deploy our solution on edge devices and measure its real-time performance.
\end{itemize}

\section{Background}

\textbf{Augmented Reality} AR systems render virtual content into a live camera stream at video rate. Here, composition and harmonization happen every frame: the renderer generates a synthetic object that must match the photometric properties of the camera feed on resource-constrained hardware (mobile GPUs, XR headsets, embedded devices). 
Despite its importance for realism, advanced color harmonization is missing from today’s mainstream AR tool-chains like ARKit, ARCore, Meta Spark, Snap Lens Studio.
The main barriers are latency and mobile compute limits.
Instead, these platforms rely on light estimation mechanisms:
\begin{itemize}
    \item Main direction of light and it's intensity
    \item Environmental cube maps \cite{greene1986environment}
    \item Spherical harmonic lighting \cite{ramamoorthi2001efficient}
    \item Global adjustments such as exposure or white-balance.
\end{itemize}

Because of computational restrictions color harmonization for AR objects remains an under-studied area. However, this field demonstrates its own unique challenges.

A typical training data for harmonization comprises of real photos or rendered scenes which then are augmented to produce synthetic composite images, i.e. every augmented image has a ground true solution. For AR composite images, with a real image on the background and virtual foreground objects, it is impossible to produce training data in the same way. Thus, even training on both real and virtual scenes leads to out-of-distribution input during inference. 

Secondly, AR composites do not suffer from the imperfect masks common in photo-editing datasets, as the rendering engine provides a pixel-perfect mask for every virtual object. On one hand, this simplifies harmonization by removing visible mask artifacts model should care about. On the other hand, models trained on real composites may struggle with inference on AR data due to the exposure bias.

\textbf{Exposure Bias in Image Harmonization} Harmonization algorithms are typically trained on datasets of composite images where foreground object masks are generated either manually or through segmentation algorithms. A key characteristic of these masks, such as those in the widely-used iHarmony4 dataset, is that they often include background pixels near the object's boundary, as shown in Fig.~ \ref{fig:exposure_bias} A. This ``pixel-imperfect'' nature can provide the model with crucial information about the ground truth harmonized background.

The presence of these boundary pixels can inadvertently teach the harmonization network to over-rely on this specific information. The model learns to predict the harmonization filter by comparing the ``leaked'' background pixels within the mask to the visible unharmonized background. This comparison is an effective shortcut, as these two background regions are far more similar in content than the foreground object and the background are.

This reliance creates a training-inference mismatch, a problem also known as the exposure bias. When the model is later applied to an object rendered by a 3D engine or cut from another image, the mask is typically ``pixel-perfect'' and contains no leaked background information. Since the model was trained to depend on this boundary information, its performance can degrade significantly when it is absent.

\begin{figure}
    \centering
    \includegraphics[width=1.0\linewidth]{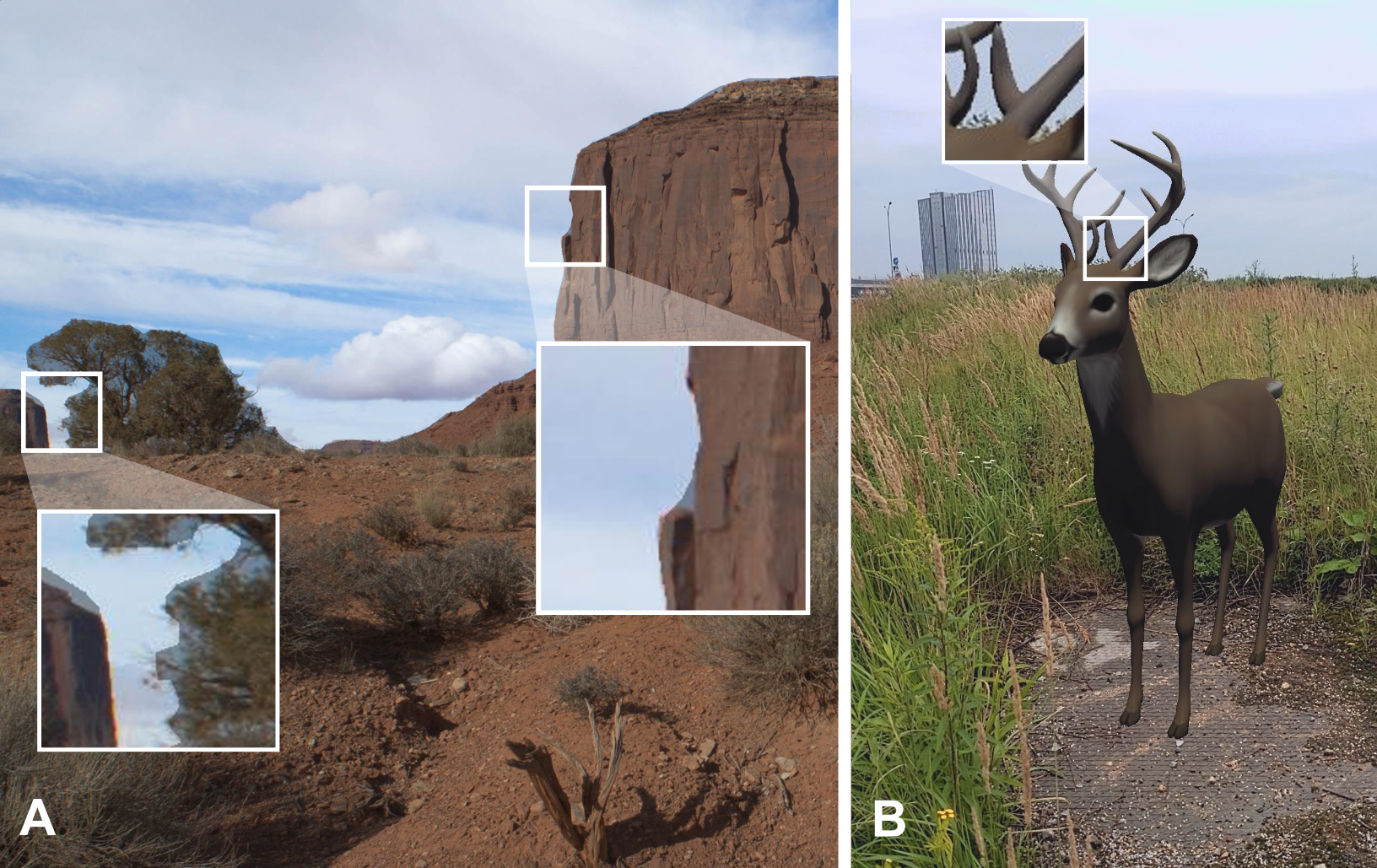}
    \caption{Exposure bias in image harmonization. Augmented image from iHarmony4 training partition. (A) Small regions near the boundary of the mask contains an information about unharmonized background and can inadvertently teach the harmonization network to over-rely on this specific information. (B) For object inserted from a 3D engine the mask is ``pixel-perfect''.}
    \label{fig:exposure_bias}
\end{figure}

\textbf{Optimal Transport} Color distributions can be modeled with continuous or discrete probability density function in RGB space. The problem of color harmonization can be seen as predicting a new color distribution for the pasted object based on the scene. That is, we want to find a new color distribution for the object as if it was in the scene originally.

We denote the original color density as $\pi_{0}$ and the predicted color density as $\pi_{1}$, with $im$ representing a scene image. The random variables $X_0 \sim \pi_{0}$ and $X_1 \sim \pi_{1}$ represent pixels sampled from their respective distributions. The harmonization problem means finding a deterministic transport map $T_{im}: \mathbb{R}^3 \rightarrow \mathbb{R}^3$ such that $T_y(X_0) = X_1$, i.e.
\begin{equation}
    \label{eq:mass_preserving_transport}
    \pi_{0}(x) = \pi_{1}(T_{im}(x)) \left|\operatorname{det} J_{T_{im}}(x)\right|,
\end{equation}
where $J_T(x)$ is the Jacobian of $T$ taken at point $x$.

\textbf{Monge's Optimal Transportation}
To select a unique map from the many that could satisfy the mass-preserving property in Eq. \ref{eq:mass_preserving_transport}, we introduce a cost function. With a standard quadratic cost, $c(x, y) = \left\|x - y\right\|^2$, the Monge problem seeks the optimal deterministic map $T^*_{im}$ that minimizes the total expected cost of transforming the source distribution to the target:
\begin{align}
    \label{eq:monge}
    \operatorname{Cost}\left[ T_{im} \right] &= \mathbb{E}\big( \left\|X_{1} - X_{0}\right\|^{2} \big) \\
    &= \int_{\mathcal{X}_0} (T_{im}(x) - x)^{2} \pi_{0}(x) dx.
\end{align}
For continuous density functions with finite second moments, a unique solution to this problem is guaranteed to exist \cite{villani2009optimal}.

When both the source and target color distributions are approximated as multivariate Gaussians, i.e., $X_0 \sim \mathcal{N}(\mu_0, \Sigma_0)$ and $X_1 \sim \mathcal{N}(\mu_1, \Sigma_1)$, this optimal transport map has a well-known closed-form linear solution, the Monge-Kantorovich Linear (MKL) filter \cite{pitie2007linear}. The optimal map $T^*_{im}$ is given by:
\begin{equation}
    \label{eq:mkl_transform}
    T^*_{im}(x) = \mu_1 + A(x - \mu_0),
\end{equation}
where the linear transformation matrix $A$ is computed as:
\begin{equation}
    \label{eq:mkl_matrix}
    A = \Sigma_0^{-1/2} \left( \Sigma_0^{1/2} \Sigma_1 \Sigma_0^{1/2} \right)^{1/2} \Sigma_0^{-1/2}.
\end{equation}

The transformation is fully characterized by the means $\mu_0, \mu_1$ and covariance matrices $\Sigma_0, \Sigma_1$.
Our method, therefore, involves training a neural network to predict these target statistics ($\mu_1, \Sigma_1$) based on the input scene, from which the MKL filter can be directly computed.

\section{Theoretical Analysis}
This section investigates the theoretical conditions under which a linear MKL map can effectively approximate a complex, non-linear color harmonization task. We provide a formal justification by bounding the approximation error, demonstrating that the MKL filter is effective when the true color transformation is smooth (i.e., has a small Lipschitz constant) and color distributions are not pathologically located at the extreme boundaries of the color gamut. 

\textbf{Model Assumptions and Preliminaries}
Let the color space be the unit cube $\mathcal{X} = [0,1]^3$. Let the unharmonized and target color distributions, $\pi_0$ and $\pi_1$, be probability measures supported on $\mathcal{X}$. We impose the following regularity conditions for our analysis.
\begin{description}
    \item[Assumption 1 (Map Regularity).] The true optimal transport map, $T^*: \mathcal{X} \to \mathcal{X}$, exists and is $L$-Lipschitz continuous for a constant $L < \infty$. The existence of such a map for a general case is non-trivial and is guaranteed under strong conditions on the measures, such as uniform log-concavity of their densities \cite{caffarelli1992regularity}. 
    \item[Assumption 2 (Distribution Regularity).] The source distribution $\pi_0$ has mean $\mu_0$ and a non-singular covariance matrix $\Sigma_0$, which means that $\Sigma_0^{1/2}$ needed for the MKL map is well-defined. All expectations $\mathbb{E}[\cdot]$ are taken with respect to $X_0 \sim \pi_0$.
\end{description}
Our algorithm approximates $T^*$ with the MKL map, $T_{\mathrm{MKL}}(x) = \mu_1 + A(x-\mu_0)$, which is the optimal map for transporting between Gaussian surrogates $\mathcal{N}(\mu_0, \Sigma_0)$ and $\mathcal{N}(\mu_1, \Sigma_1)$ \cite{peyre2019computational}. Since the MKL map could potentially map part of the density outside of the unit cube, we apply clipping operation $\hat{T}_{\mathrm{MKL}} := \Pi_{\mathcal{X}} \circ T_{\mathrm{MKL}}$, where $\Pi_{\mathcal{X}}$ is the  Euclidean projection onto $\mathcal{X}$, i.e. clipping transformed colors to the valid $[0,1]^3$ gamut. We proof that the clipping operation coincides with the Euclidean projection in Lemma \ref{lem:clip-is-proj}.

\newcommand{\clip}{\operatorname{clip}}

\begin{lemma}[Clipping equals Euclidean projection]\label{lem:clip-is-proj}
Let the \emph{clipping operator} $\clip:\mathbb{R}^{d}\!\to\![0,1]^{d}$ be defined component-wise by
\begin{equation}
(\clip(z))_j
   = \min\{1,\max\{0,z_j\}\},
   \qquad j=1,\dots,d.
\end{equation}
Then for every $z\in\mathbb{R}^{d}$ the vector $y=\clip(z)$ is the unique Euclidean projection of $z$ onto the cube $\mathcal{X}=[0,1]^d$; that is,
$\clip(z)=\Pi_{\mathcal{X}}(z)$.
\end{lemma}


Let us note that any projection operator is $1$-Lipschitz since the distance between any two points after projection is never greater than their distance before projection. We seek to bound the expected squared error $\mathcal{E} := \mathbb{E}[ \| \hat{T}_{\mathrm{MKL}}(X_0) - T^*(X_0) \|^2 ]$.

\begin{theorem}[Error Bound for L-Lipschitz Color Maps]
Let Assumptions 1 and 2 hold. The total error $\mathcal{E}$ is bounded as:
\begin{equation}
    \mathcal{E} \le 2\mathcal{E}_{clip} + 2\mathcal{E}_{lin},
\end{equation}
where the clipping error is $\mathcal{E}_{clip} := \mathbb{E} [ \| T_{MKL}(X_0) - \hat{T}_{MKL}(X_0) \|^2 ]$, and the linearity error, $\mathcal{E}_{lin}$, is bounded by:
\begin{equation}
    \mathcal{E}_{lin} \le 2 B^2 + 2 (\|A\|_{op} + L)^2 \cdot \mathrm{tr}(\Sigma_0).
\end{equation}
Here, $B = |\mu_1 - T^*(\mu_0)|$ is a bias term, $||A||_{op}$ is the spectral norm of the MKL matrix, i.e. its largest singular value, which depends only on source and target distribution covariances, and $\mathrm{tr}(\Sigma_0)$ is the trace of the source covariance matrix.
\end{theorem}
\begin{proof}[Proof Sketch]
Insert and subtract $T_{\mathrm{MKL}}(X_0)$ inside the norm, then use $\|u+v\|^{2}\!\le 2\|u\|^{2}+2\|v\|^{2}$ to obtain $\mathcal{E}\le 2\mathcal{E}_{\text{clip}}+2\mathcal{E}_{\text{lin}}$.  
To bound $\mathcal{E}_{\text{lin}}$, add and subtract $T^{*}(\pi_{0})$, invoke the $L$-Lipschitz property of $T^{*}$ together with $\|A\|_{\text{op}}$ to control $\|T_{\mathrm{MKL}}(x)-T^{*}(x)\|$ by a bias term $B$ plus a scaled deviation from the mean.  
Squaring, taking expectations, and using $\mathbb{E}\|X_{0}-\mu_{0}\|^{2}=\mathrm{tr}(\Sigma_{0})$ yields the claimed bound.
\end{proof}

This theorem bounds the approximation error, linking it to the Lipschitz constant $L$ of the true optimal transport map and the tail probability $\mathcal{E}_{clip} \le d \cdot \mathbb{P}[T_{\mathrm{MKL}}(X_0) \notin \mathcal{X}]$. The latter bound holds due to the following Lemma \ref{lem:clip-tail}.

\begin{lemma}[Tail–probability bound for the clipping error]%
\label{lem:clip-tail}
Let $\Pi_{\mathcal{X}}=\clip(\cdot)$ be the Euclidean projection onto
$\mathcal{X}=[0,1]^d$, defined as
\begin{align}
    (\clip(z))_j \;:=\; \min\bigl\{\,1,\;\max\{0,\;z_j\}\bigr\},
\end{align}
for $j=1,\dots,d,\;z\in\mathbb{R}^d$ and define
\begin{align}
\mathcal{E}_{\text{clip}}
\;:=\;
\mathbb{E}\!\bigl[\|Z-\Pi_{\mathcal{X}}Z\|^2\bigr],
\qquad Z\in\mathbb{R}^d.
\end{align}
Then
\begin{align}
\mathcal{E}_{\text{clip}}
  &\le
  d\,\mathbb{P}[Z\notin\mathcal{X}],
  \label{eq:clip-bound}
\end{align}
and for our application with $Z=T_{\mathrm{MKL}}(X_0)$ and $d=3$,
\begin{align}
\mathcal{E}_{\text{clip}}
  &\le
    3\,\mathbb{P}\bigl[T_{\mathrm{MKL}}(X_0)\notin\mathcal{X}\bigr].
  \tag{\ref{eq:clip-bound}$'$}
\end{align}
\end{lemma}


In practice, color harmonization is a smooth process that maps adjacent colors to nearby ones, ensuring a small Lipschitz constant. This, combined with the fact that color distributions in natural images are rarely concentrated at the gamut boundaries, makes the clipping error negligible and justifies our linear approximation.

However, this is not the case when the harmonization task considers dark objects. Their color distribution is concentrated around a corner of $\mathcal{X}$. In our experiments we indeed observe that processing of too dark objects often leads to implausible results.

\section{Method}

\textbf{Ground-Truth Filter Generation} To verify whether our theoretical assumptions hold (i.e. a simple linear filter is powerful enough for image harmonization) we first have computed standard set of iHarmony metrics for so-called \textit{Ideal Linear OT filter}. This is not a predictive model, but rather a theoretical ceiling: for each image in the iHarmony dataset, we compute and store the exact MKL transform (Eq.~\ref{eq:mkl_transform}) that optimally maps the color distribution of the augmented unharmonized distribution to its ground-truth. This ideal linear filter achieves a Mean Squared Error (MSE) $\sim 7.0$. which is quite low compared to the state-of-the-art results for this benchmark, see Tab.~\ref{tab:metrics_mse}. High performance of linear MKL filters suggests that real harmonization maps are often Lipschitz and are mostly not concentrated near the gamut boundaries.
Therefore, we use these pre-computed vectors as the primary supervisory signal during training.

\textbf{Loss function} At this point, we have two options. The first one is to train the encoder network $\text{Model}(\cdot)$ to predict statistics $\mu_1, \Sigma_1$ of the target harmonized  distribution $\pi_{1}$. Since the input $\mu_0, \Sigma_0$ is known, the resulting filter is computed straightforwardly following Eq.~\ref{eq:mkl_transform}. The second option is to directly predict $A$ given by Eq.~\ref{eq:mkl_matrix} and shift $S$
\begin{equation}
    \label{eq:mkl_shift}
    S = \mu_1 - A \mu_0
\end{equation}

In our experiments we test both objectives for predicting tuples $[\mu_1, \Sigma_1]$ or $[A, S]$. We find that the latter one is more robust to the prediction inaccuracies and yields lower loss overall. 

\begin{equation}
    \label{eq:labels_loss}
     L_{labels} = || \text{Model}(im) - [A, S] ||_{1}
\end{equation}

In addition to this standard labels MSE loss, we found it beneficial to include the content L1 per-pixel loss for $[A', S']$ predicted by the model

\begin{equation}
    \label{eq:content_loss}
    L_{content} = || M * X_0 - M * (X_0 \cdot A' + S') ||_{1}
\end{equation}

where $M$ is binary mask, $X_0$ is flattened composite image pixels, $*$ denotes element-wise product and $\cdot$ is matrix product.

Since our problem is ill-posed, there is no unique MKL filter that satisfies the harmonization objective. Minimizing the expected $L_{2}$ error would therefore force the estimator to produce the unique arithmetic mean of all possible MKL solutions, which may not correspond to any actual filter. In contrast, we employ the $L_{1}$ loss, which is minimized by any valid solution and is, in a sense, less restrictive. The $L_{1}$ loss allows the network to converge to a sharper solution rather than the overly smoothed compromise imposed by $L_{2}$. Metrics for $L_{labels}$ with $L_{1}$ and $L_{2}$ losses are present in the Tab.~\ref{tab:metrics_mse}.

The question arises, why not to stick to the $L_{content}$ only? In practice, when training with just $L_{content}$, model learns filters that are close to identity transformation. The possible explanation of this behavior is that in the absence of MKL guiding signal the model experiences mode collapse similar to the case described in PCT-Net training, where authors introduce contrastive loss to avoid filter collapse. So our final loss takes the form:

\begin{equation}
    \label{eq:total_loss}
    L_{total} = L_{labels} + \alpha L_{content}
\end{equation}
We discuss parameters choice in the experiments section.

\begin{figure}[t]
    \centering
    \includegraphics[width=1\linewidth]{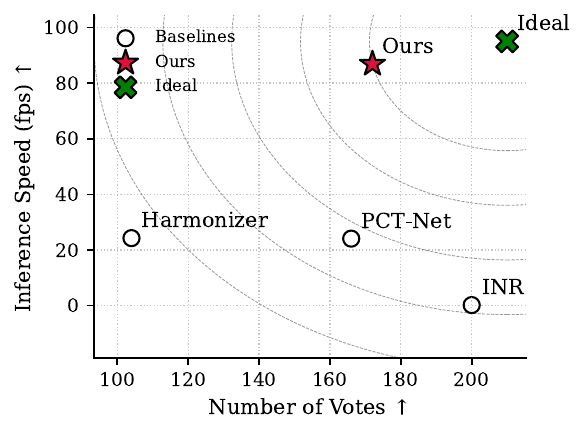}
    \caption{Mean opinion score versus inference speed calculated for images with 1080x2204 resolution from ARCore data. Data was processed on NVIDIA RTX 4060Ti GPU.}
    \label{fig:votes_vs_speed}
\end{figure}

\section{Experiments and Metrics}

\textbf{Datasets and Metrics} We train our model on iHarmony4 dataset \cite{cong2020dovenet}, which contains synthesized composite images, foreground masks of composite images and corresponding real images.

For evaluation we use approach, standard for the color harmonization area: mean squared error (MSE), peak signal-to-noise ratio (PSNR) and foreground MSE (fMSE) calculated on the iHarmony test set in 256x256 resolution.

However, since we initially aimed to study harmonization in augmented reality, the central place in our evaluation is taken by AR-specific data.

\begin{figure}
    \centering
    \includegraphics[width=1.0\linewidth]{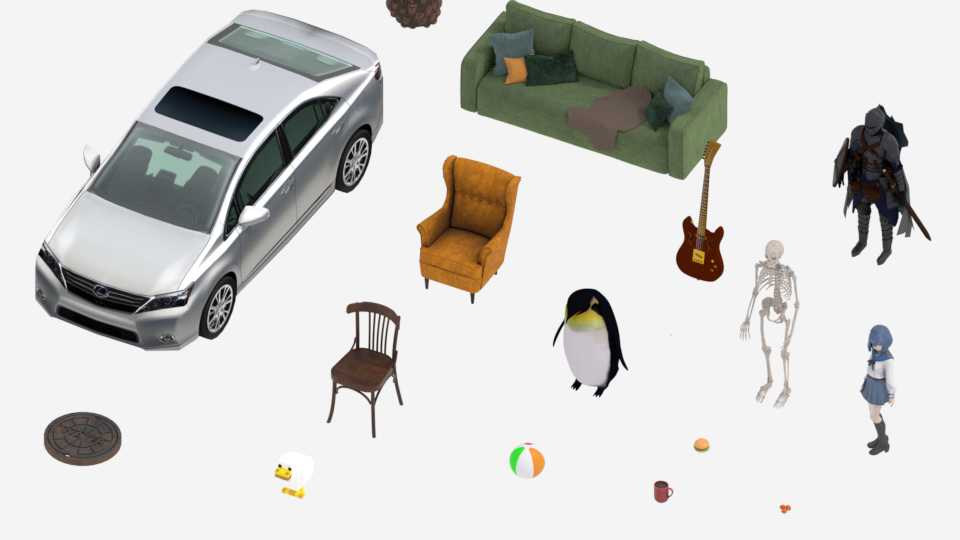}
    \caption{The 3D objects used in our experiments feature different sizes and textures.}
    \label{fig:3d_objects}
\end{figure}

\textbf{ARCore Evaluation Set} AR images are real composite images (i.e. produced without augmentation and thus they have no ground truth), making the standard metrics inapplicable. Moreover, AR data is scarce by itself. While one can find many options with fully rendered scenes like HVIDIT \cite{guo2021intrinsic} or Hypersim \cite{roberts2021Hypersim}, there are no collections of relevant AR images with per-pixel masks available in the open source. 

For this reason, we modify sample ARCore \footnote{AR platform for Android developed and supported by Google} application to turn it into basic data-gathering tool and collect the first small-scaled dataset of this kind. Currently the set contains 327 pairs of composite-mask images, captured in the wild, as shown in Fig.~\ref{fig:pair_of_balls}. It features various indoor and outdoor scenes, different day time, weather and lighting conditions. 3D objects used in our experiments are collected from open source and converted into the proper format. Selected object does not follow the common style and varies in sizes in textures as shown in Fig.~\ref{fig:3d_objects}. Please refer to the Supplementary for model credits and licenses. The rendering pipeline of application we used for data accumulation and model testing is depicted in Fig.~\ref{fig:app_scheme}. Crucially, all masks are obtained directly from the rendering engine.

\textbf{Baselines} We compare our method against three harmonization models described in the introduction section: Harmonizer \cite{ke2022harmonizer}, PCT-Net \cite{guerreiro2023pct},  INR-Harmonization \cite{chen2023dense}. Also, for iHarmony4 dataset, we include a comparison with classical color transfer between foreground and background pixels \cite{reinhard2001color}. 
All baseline models are evaluated using  MSE, PSNR, and fMSE metrics, results are shown in Tab.~\ref{tab:metrics_mse}.

\textbf{User Study on ARCore Data} For the user study, participants were shown the results of four harmonization methods -- Harmonizer, PCT-Net, INR-Harmonization and our MKL encoder -- applied to the same composite AR image and asked the question: ``Which image is more natural and realistic?". Each image set was shuffled. We gathered 20 participants, each of whom graded around 30 sets of these 4-way comparisons as demonstrated in Fig.~\ref{fig:human_study}, resulting in total 642 grades (see Fig.~\ref{fig:method_votes})

\begin{figure}[t]
    \centering
    \includegraphics[width=1.0\linewidth]{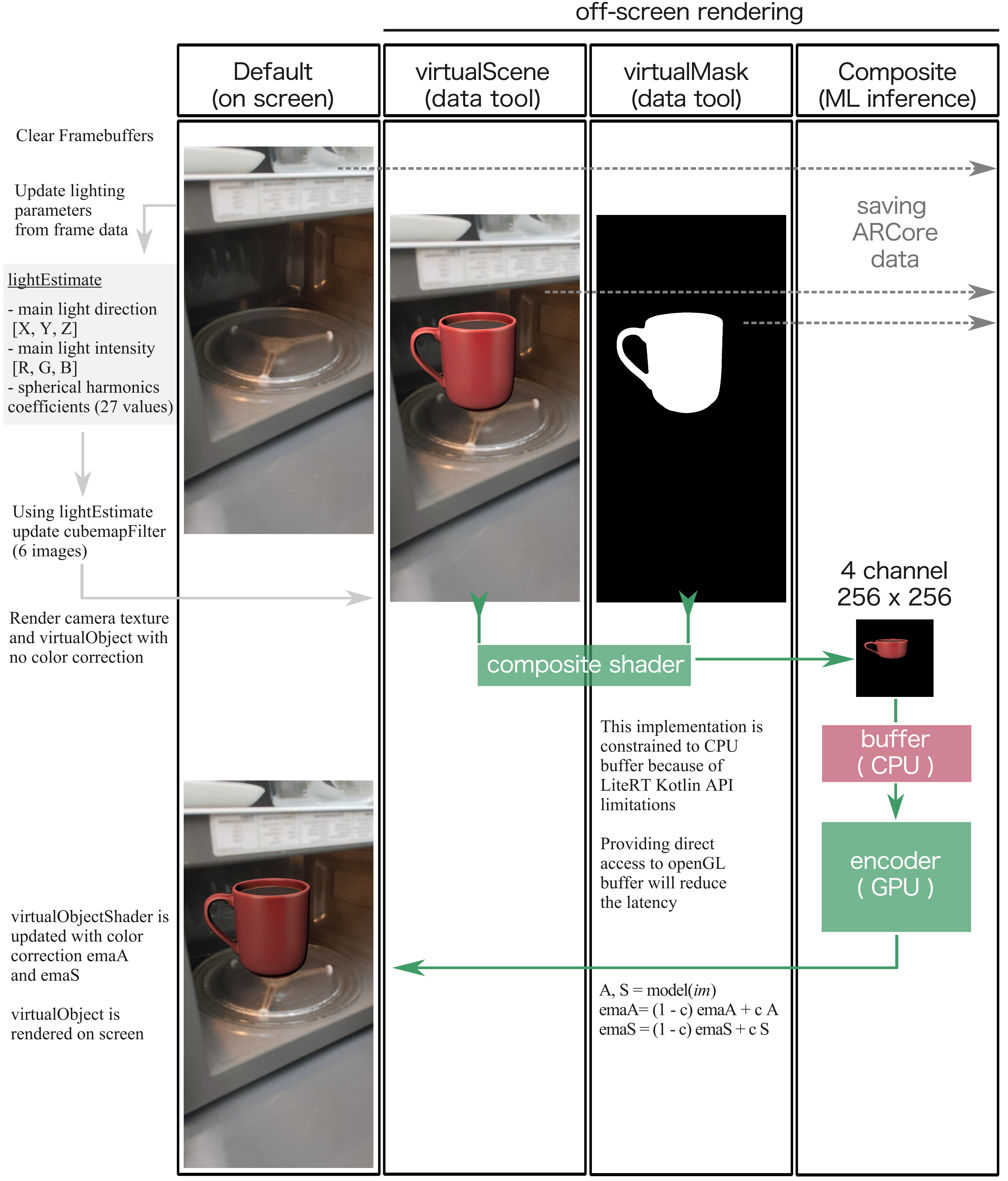}
    \caption{The general scheme of main rendering loop demonstrates four openGL Framebuffers. Default one performs on-screen rendering, others are used for off-screen rendering and auxiliary tasks, such as data collection. Limitations of LiteRT Next Kotlin introduces computational overhead due to copying model input through the CPU. Masks are rendered concurrently with foreground objects, obtained directly from the rendering engine and are saved on user's demand.}
    \label{fig:app_scheme}
\end{figure}

\textbf{Encoder} We use an EfficientNet-B0 \cite{tan2019efficientnet}, also used in PCT-Net and Harmonizer, chosen for its efficiency on mobile hardware. We modify its first convolutional layer to accept a 4-channel input (RGB + mask), allowing it to process both color and spatial information simultaneously. The network is trained on $256 \times 256$ inputs and its final regression head outputs a 12-dimensional vector corresponding to the parameters of the approximate MKL filter.

\textbf{Training} The model is trained for 210 epochs on an NVIDIA RTX 4090 using the Adam optimizer \cite{kingma2014adam} with a learning rate starting at $1\mathrm{e}{-3}$ and decreasing at 50 epochs to $1\mathrm{e}{-4}$ and at 110 epochs to $5\mathrm{e}{-5}$ with a batch size of 64. We employ a hybrid loss function to ensure perceptual quality:
\begin{equation}
    L_{total} = L_{labels} + \alpha \cdot L_{content}
\end{equation}
We set the content loss weight $\alpha=10$. This hybrid loss is critical for preventing the model from collapsing to a simple identity transformation as discussed in the Method section.

\textbf{Perceptual Evaluation Results} Given that image harmonization is an ill-posed problem, and in light of the exposure bias, we argue that perceptual evaluation by human observers is more insightful. 

The results of our user study, summarized in Fig.~\ref{fig:method_votes}, reveal two key insights. First, our method is rated perceptually on par with the leading baselines on real AR data. It proves that for pixel-perfect masks, where edges inconsistency is not an issue, simpler solution can be effective. 

Second, this exposes the weakness of relying on MSE based metrics; while the INR model has notably lower MSE score compared with PCT-Net, human observers rate its perceptual quality much higher, making MOS a more reliable metric. We suggest that the difference between the human scores and the metric in Tab.~\ref{tab:metrics_mse} may arise due to the training-inference mismatch discussed in the background section. The qualitative comparison is given in Fig.\ref{fig:qualitative_comparsion}.

Furthermore, Tab.~\ref{tab:metrics_its} and Fig.~\ref{fig:votes_vs_speed} illustrate the essential trade-off between perceptual quality (MOS) and inference speed. Our method offers both the high perceptual score and the fastest performance.

\textbf{Inference on an Edge Device} We evaluate the on-device inference performance of our method using the Google Pixel 4a and Google Pixel 7. We handle model conversion via the LiteRT Next Kotlin API and run the model on each frame of the rendering loop. Frame rate performance is measured on devices before and after turning on online harmonization, resulting in 12 to 15 fps, see Fig.~\ref{fig:frame_rate}. However, our tested implementation makes two unnecessary passes thought CPU buffers (see Fig.~\ref{fig:app_scheme}). Implementing zero-copy routines could potentially double this framerate to 24 - 30 fps range \cite{litert2025}.

\begin{table}
\caption{Quantitative results on the iHarmony4 256x256 dataset. Ideal Linear OT filter represents reference values calculated based on ground true images. Even though these values are not achievable in practice, they prove linear filters may be capable enough.}
\label{tab:metrics_mse}
\begin{tabular}{lccc}
\toprule
Method & MSE & PSNR & fMSE \\
\midrule
Ideal Linear OT & 7.6 ± 0.2 & 43.6 ± 0.1 & 45.9 ± 0.9 \\
\midrule
PCT-Net & 29.1 ± 0.9 & 38.0 ± 0.1 & 201 ± 4 \\
Harmonizer & 40.1 ± 1.2 & 36.6 ± 0.1 & 258 ± 5 \\
Ours $L_{1}$  & 65.0 ± 1.6 & 34.1 ± 0.1 & 438 ± 7 \\
Ours $L_{2}$ & 66.3 ± 1.7 & 33.9 ± 0.1 & 451 ± 7 \\
INR  & 67.2 ± 1.8 & 35.3 ± 0.1 & 392 ± 7 \\
CT & 284 ± 6.9 & 27.5 ± 0.1 & 1836 ± 23\\
\midrule
Unharmonized & 182 ± 5 & 31 ± 0.1 & 984 ± 17 \\
\bottomrule
\end{tabular}
\end{table}

\begin{table}[t]
\centering
\caption{Performance test on different resolutions measured in iterations per second. Experiment was carried out on RTX 4060Ti.}
\label{tab:metrics_its}
\resizebox{\columnwidth}{!}{%
\begin{tabular}{lcccc}
\toprule
Method & 256$\times$256 & 512$\times$512 & 1024$\times$2048 & 4096$\times$4096 \\
\midrule
Ours         & 175.01 & 166.76 & 137.21 &  40.85 \\
DoveNet      & 123.39 &   –    &   –    &   –    \\
PCT-Net      & 104.57 &  98.65 &  63.74 &  11.84 \\
Harmonizer   &  95.01 &  89.82 &  47.63 &   7.45 \\
INR          &   6.35 &   3.22 &   0.81 &   0.12 \\
\bottomrule
\end{tabular}%
}
\end{table}

\begin{figure}[ht]
    \centering
    \includegraphics[width=0.75\linewidth]{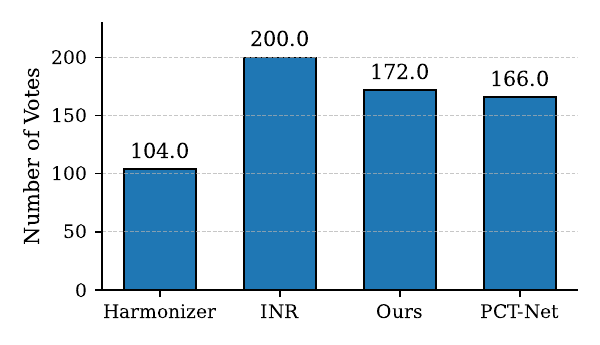}
    \caption{Mean opinion score on ARCore data. Results discussed in section ARCore Dataset and user Study.
    }
    \label{fig:method_votes}
\end{figure}

\begin{figure}[ht!]
    \centering
    \includegraphics[width=1\linewidth]{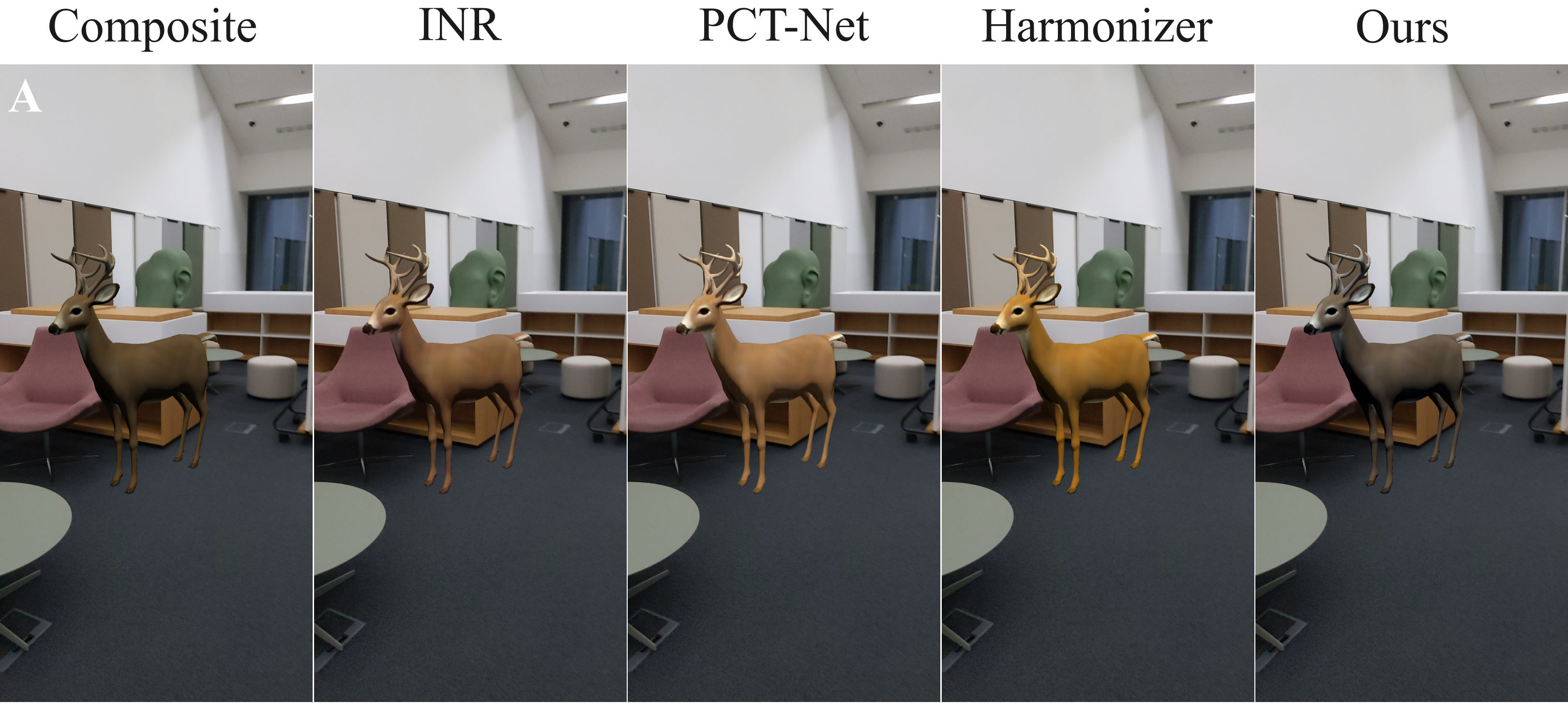}
    \includegraphics[width=1\linewidth]{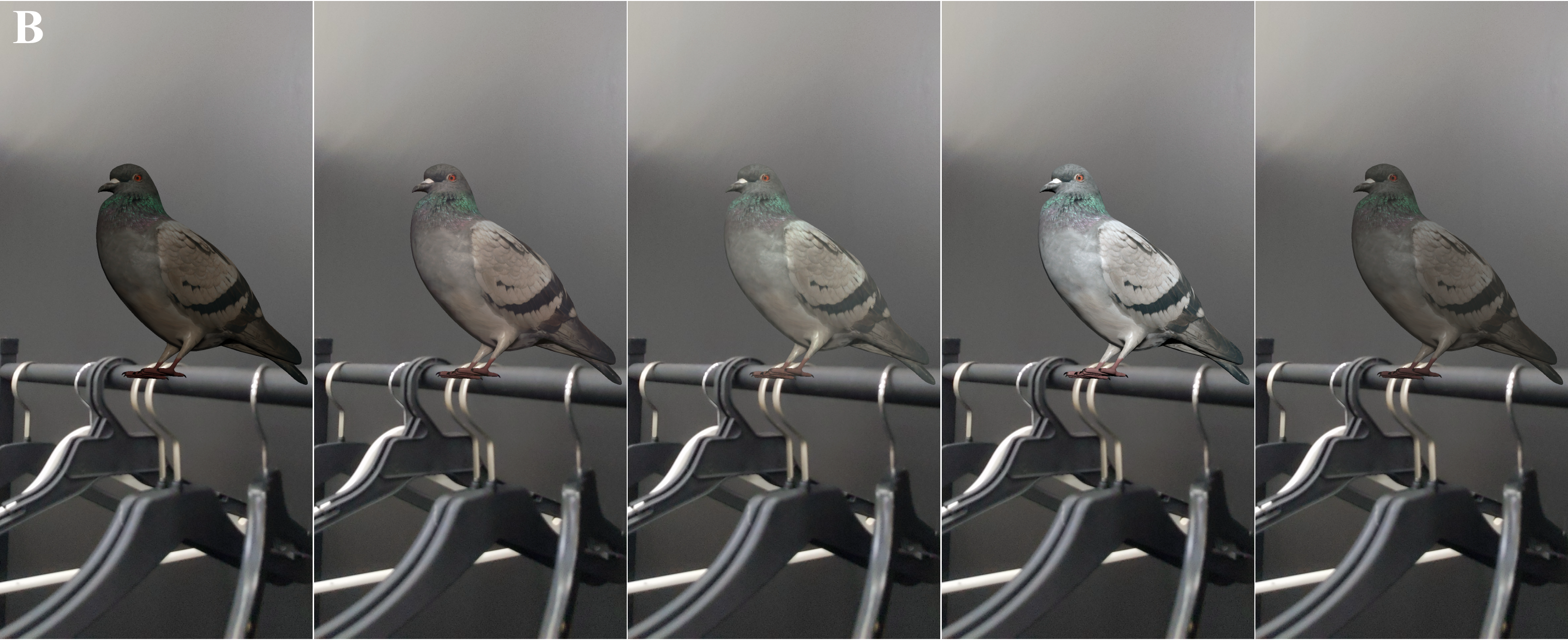}
    \includegraphics[width=1\linewidth]{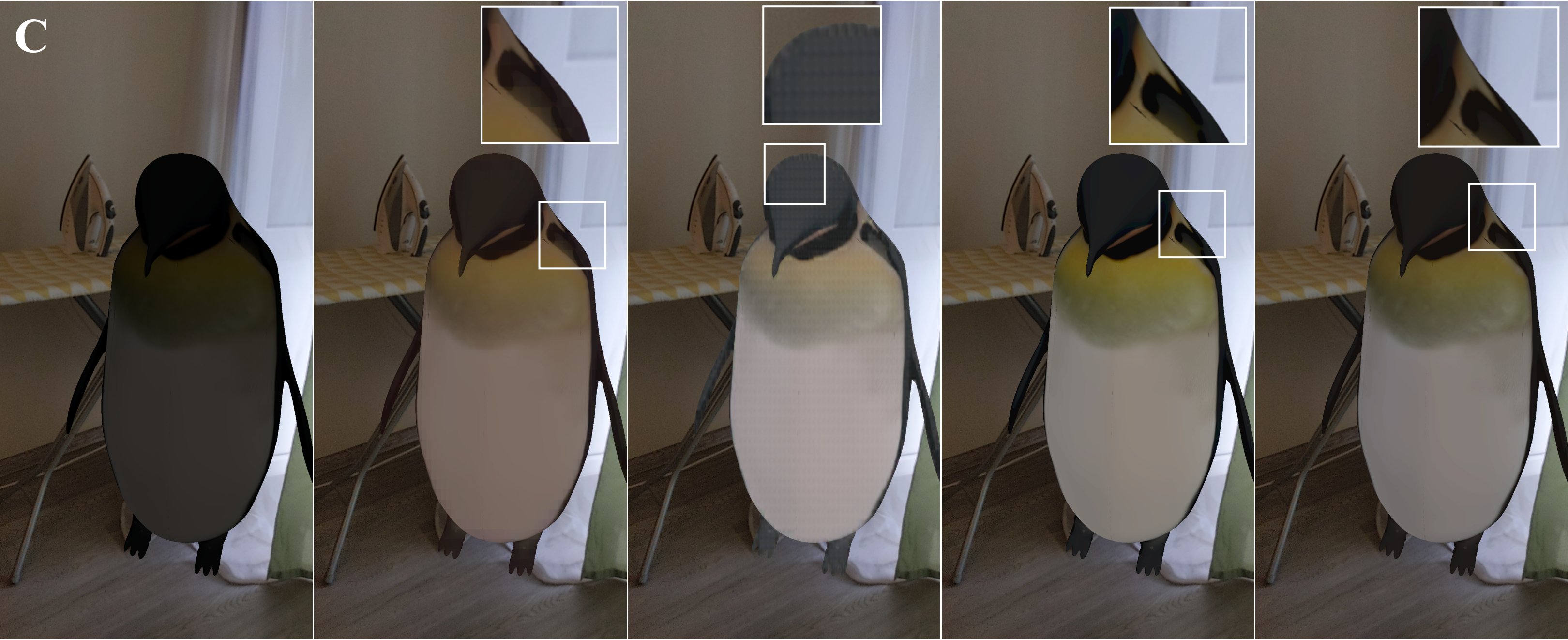}
    \includegraphics[width=1\linewidth]{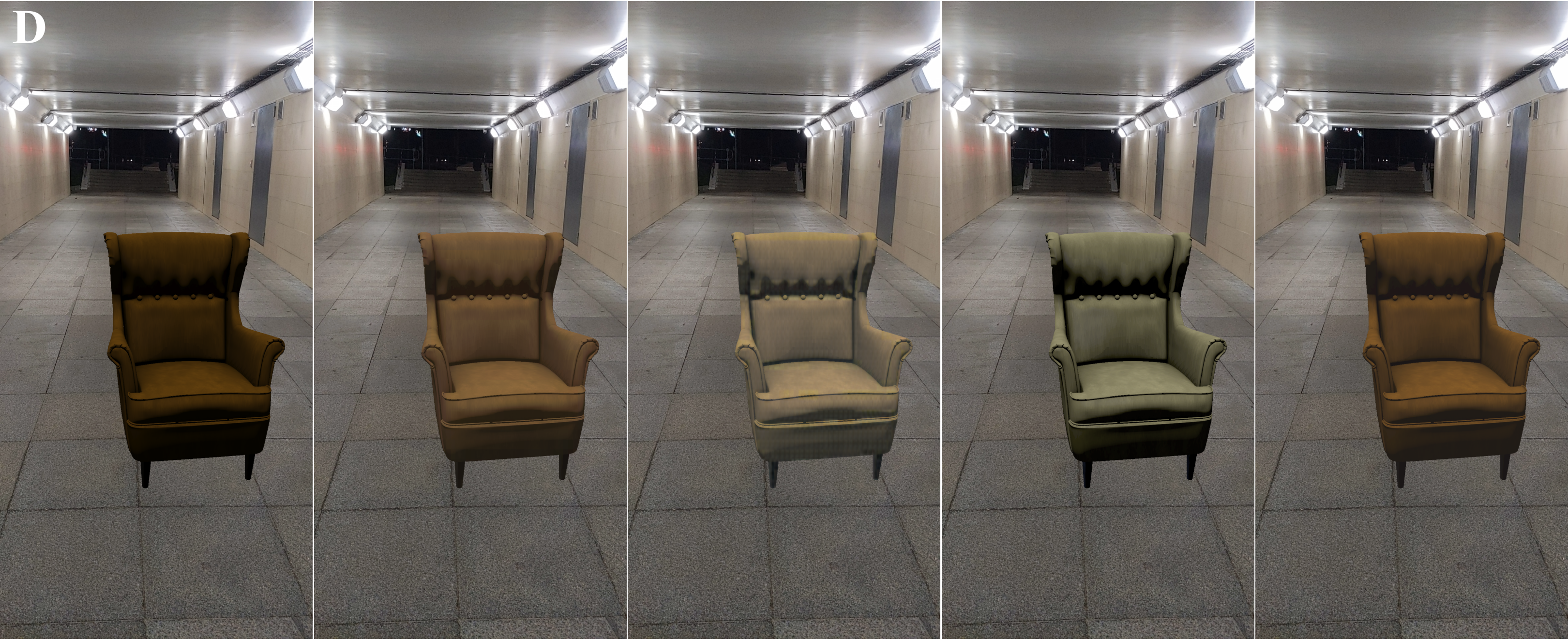} 
    \includegraphics[width=1\linewidth]{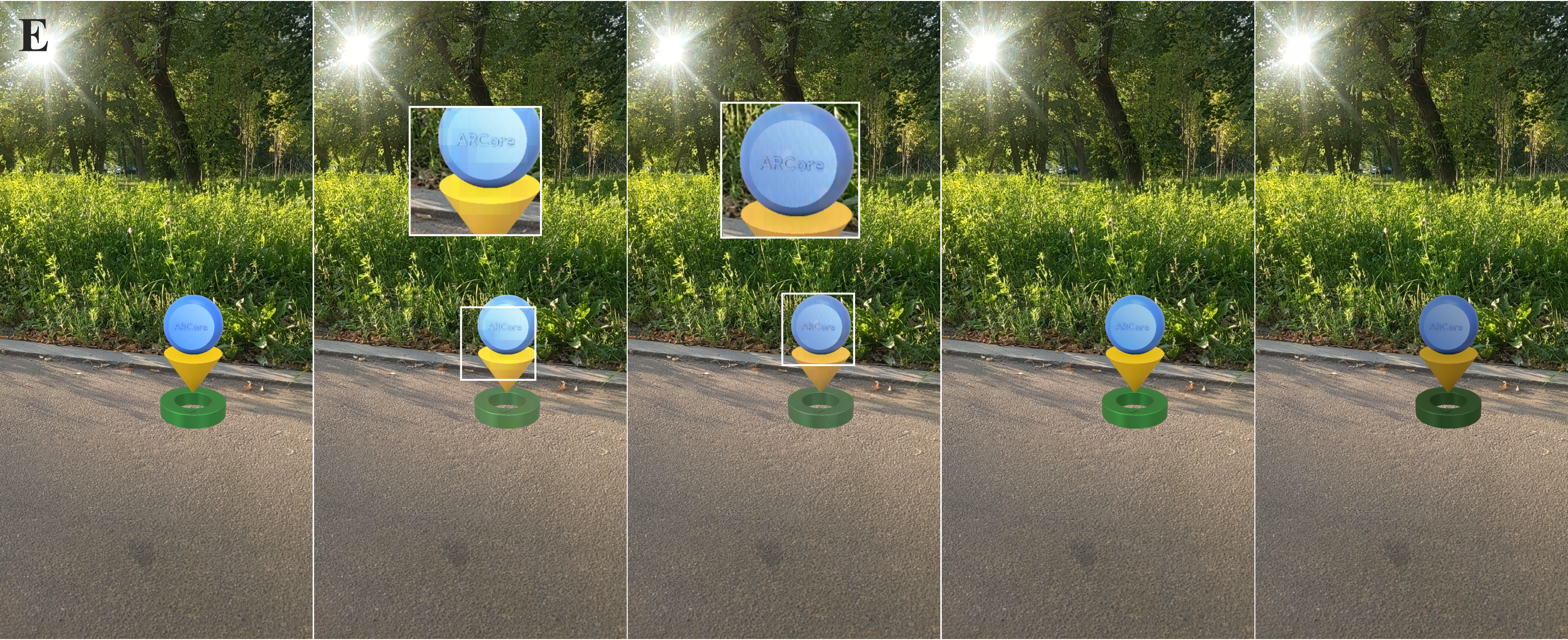}
    \includegraphics[width=1\linewidth]{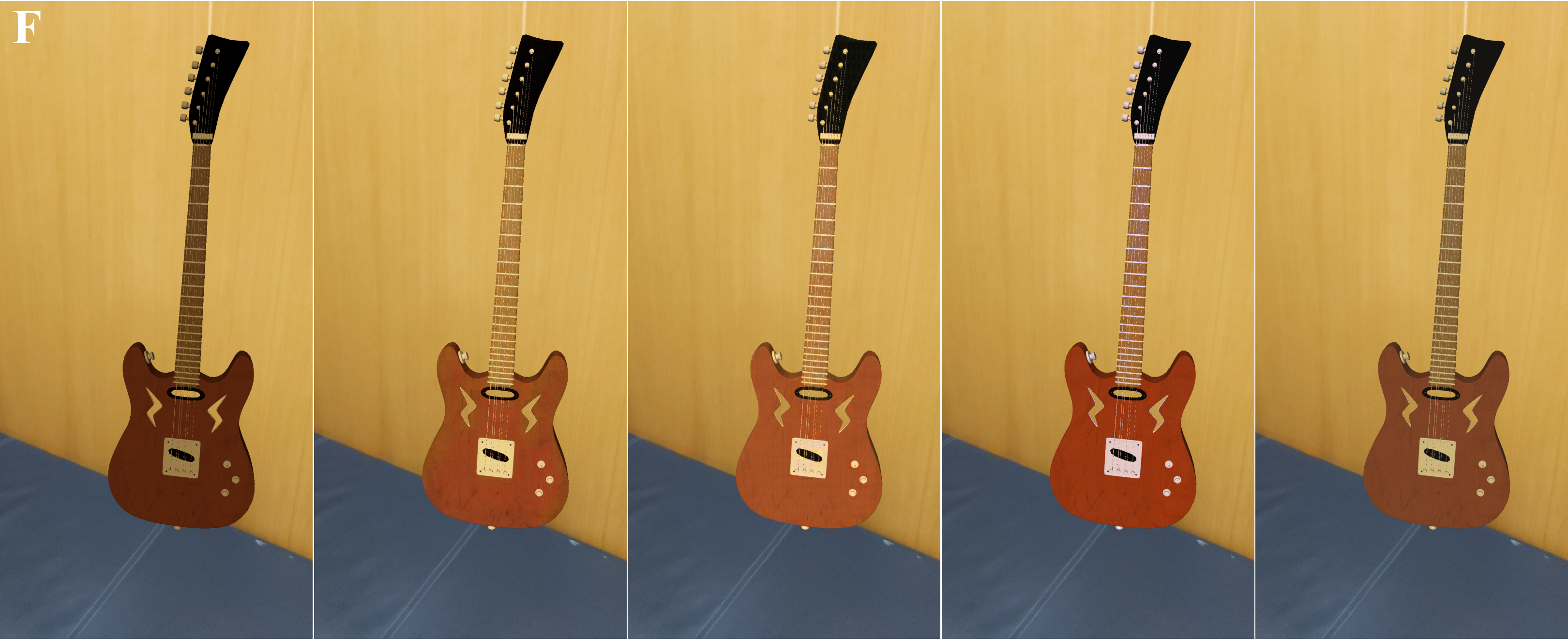}

    \caption{Qualitative comparison with baselines.}
    \label{fig:qualitative_comparsion}
\end{figure}

\begin{figure}[h!]
    \centering
    \includegraphics[width=1\linewidth]{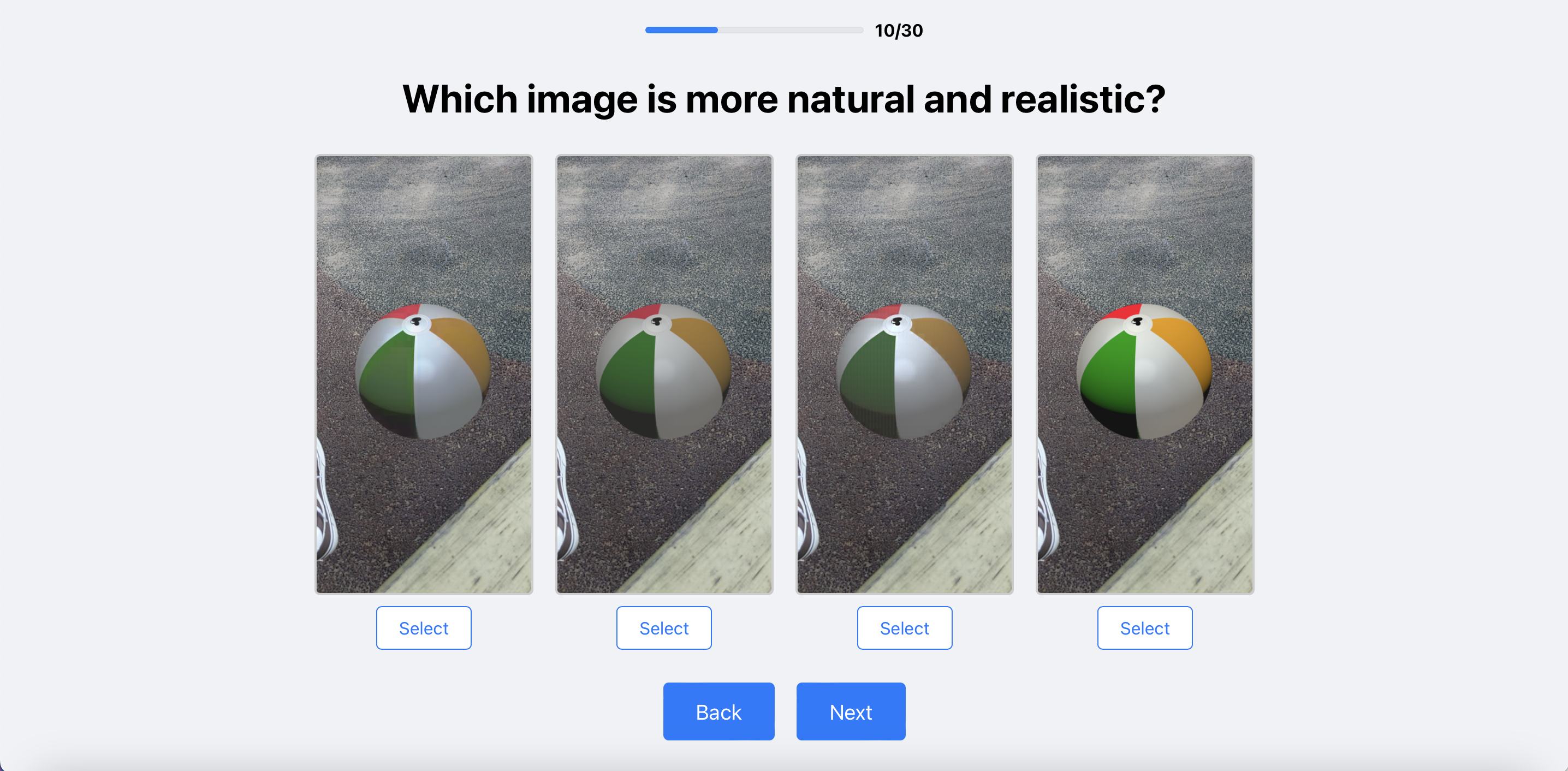}
    \caption{User interface of the labeling page.}
    \label{fig:human_study}
\end{figure}

\begin{figure}[ht!]
    \centering
    \includegraphics[width=0.75\linewidth]{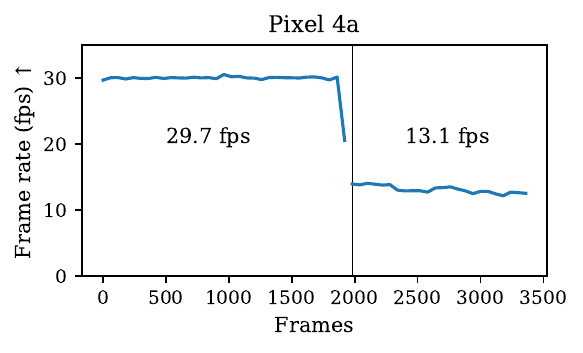}
    \includegraphics[width=0.75\linewidth]{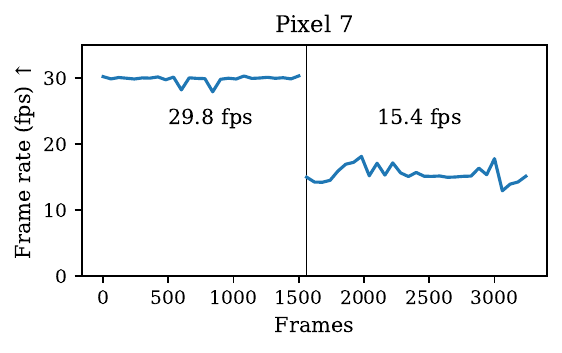}
    \caption{Frame rate profiles from the device before and after harmonization is turned on.}
    \label{fig:frame_rate}
\end{figure}

\subsection{Discussion and Limitations}
For high-resolution images, we observe that dense prediction models often produce corrupted results, leading to coarse or degraded outputs. In contrast, Harmonizer and our method do not exhibit these defects, as they are filter-based methods and thus avoid upscaling-related artifacts.
In Fig.~\ref{fig:qualitative_comparsion}C and E, one can observe coarse stripes in PCT-Net and pixelated areas resembling JPEG artifacts in INR. Notably, these issues are not present in low-resolution settings. However, all models exhibit certain biases in their predictions—for instance, Harmonizer tends to increase image brightness, while our predictions tend to be darker.
Another limitation of our method is that it is not designed as a final solution for video harmonization. Since our model is not trained on video data, sequential predictions may vary significantly across frames. To mitigate this, we apply an exponential moving average, as illustrated in Fig.~\ref{fig:app_scheme}. However, video harmonization remains a challenging problem overall.

\section{Conclusion}
In this work, we introduced a lightweight and efficient method for color harmonization tailored for real-time augmented reality applications. We framed the harmonization task as an optimal transport problem and grounded our method in classical OT theory by training an encoder to predict the parameters of a Monge-Kantorovich Linear filter. Furthermore, we provided a theoretical analysis that justifies the use of a linear map by bounding its approximation error, showing it is effective for the smooth color transforms typical of harmonization.

Our central contribution is not only the model itself but also a pioneering analysis of AR-specific harmonization challenges. We critically evaluated how state-of-the-art methods are performed and identified a possible ``exposure bias'' in the standard iHarmony4 dataset, which may cause metrics like MSE to misrepresent perceptual quality. To address this, we created and introduced the ARCore dataset, a new benchmark with pixel-perfect masks for realistic AR evaluation.

\bibliography{aaai2026}

\clearpage
\onecolumn
\makeatletter
\twocolumn[
\begin{@twocolumnfalse}
\section*{Supplementary Material for:\\ Lightweight Optimal-Transport Harmonization on Edge Devices}
\vspace{2.5em}
\end{@twocolumnfalse}
]
\makeatother
\section{Experiments and Metrics}
\textbf{iHarmony4 Dataset Preprocessing} During the metrics calculation, we observed that iHarmony4 composite images exhibit high-frequency artifacts, as illustrated by the difference image in Figure \ref{fig:iharmony_artifacts}. According to the dataset generation procedure, unmasked regions of a composite image should be exactly the same as in the corresponding ground-truth image. However, we found substantial differences that affect the MSE and PSNR metrics. To exclude this noise, we re-assembled the iHarmony4 dataset by replacing the pixels from unmasked regions of composites with pixels taken from their ground-truth images. The resulting artifact-free set is saved in PNG format and named `iHarmony4-clean`. We evaluate all baselines on this cleaned dataset to ensure a fair comparison.

\begin{figure}[h]
    \centering
    \includegraphics[width=0.7\linewidth]{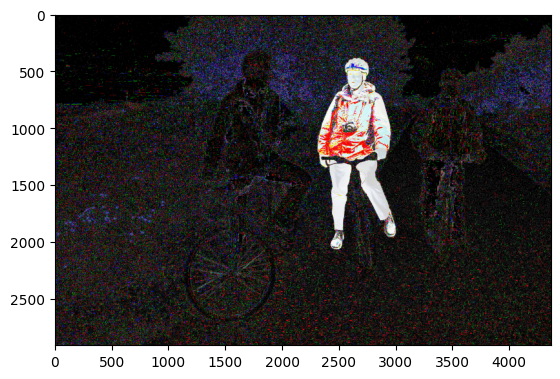}
    \caption{High-frequency artifacts are visible in a difference image between an original iHarmony4 composite and its corresponding real image.}
    \label{fig:iharmony_artifacts}
\end{figure} 

Additional comparison with baselines on the ARCore data is given in Fig. \ref{fig:qualitative_comparsion_supp}, and for comparison with ground-truth, refer to Fig. \ref{fig:qualitative_comparsion_iharmony}.

\section{Theoretical Analysis}
This appendix provides the complete proofs of Lemma~1, Theorem~1 and Lemma~2 stated in Section~6 of the main paper. We adopt the same notation and Assumptions~1--2 introduced there. For convenience we restate each result below together with its proof.

\textbf{Model Assumptions and Preliminaries}
Let the color space be the unit cube $\mathcal{X} = [0,1]^3$. Let the unharmonized and target color distributions, $\pi_0$ and $\pi_1$, be probability measures supported on $\mathcal{X}$. We impose the following regularity conditions for our analysis.
\begin{description}
    \item[Assumption 1 (Map Regularity).] The true optimal transport map, $T^*: \mathcal{X} \to \mathcal{X}$, exists and is $L$-Lipschitz continuous for a constant $L < \infty$. The existence of such a map for a general case is non-trivial and is guaranteed under strong conditions on the measures, such as uniform log-concavity of their densities \cite{caffarelli1992regularity}. 
    \item[Assumption 2 (Distribution Regularity).] The source distribution $\pi_0$ has mean $\mu_0$ and a non-singular covariance matrix $\Sigma_0$, which means that $\Sigma_0^{1/2}$ needed for the MKL map is well-defined. All expectations $\mathbb{E}[\cdot]$ are taken with respect to $X_0 \sim \pi_0$.
\end{description}
Our algorithm approximates $T^*$ with the MKL map, $T_{\mathrm{MKL}}(x) = \mu_1 + A(x-\mu_0)$, which is the optimal map for transporting between Gaussian surrogates $\mathcal{N}(\mu_0, \Sigma_0)$ and $\mathcal{N}(\mu_1, \Sigma_1)$ \cite{peyre2019computational}. Since the MKL map could potentially map part of the density outside of the unit cube, we apply clipping operation $\hat{T}_{\mathrm{MKL}} := \Pi_{\mathcal{X}} \circ T_{\mathrm{MKL}}$, where $\Pi_{\mathcal{X}}$ is the  Euclidean projection onto $\mathcal{X}$, i.e. clipping transformed colors to the valid $[0,1]^3$ gamut. We proof that the clipping operation coincides with the Euclidean projection in Lemma \ref{lem:clip-is-proj}.


\begin{lemma}[Clipping equals Euclidean projection]\label{lem:clip-is-proj-supp}
Let the \emph{clipping operator} $\clip:\mathbb{R}^{d}\!\to\![0,1]^{d}$ be defined component-wise by
\begin{equation}
(\clip(z))_j
   = \min\{1,\max\{0,z_j\}\},
   \qquad j=1,\dots,d.
\end{equation}
Then for every $z\in\mathbb{R}^{d}$ the vector $y=\clip(z)$ is the unique Euclidean projection of $z$ onto the cube $\mathcal{X}=[0,1]^d$; that is,
$\clip(z)=\Pi_{\mathcal{X}}(z)$.
\end{lemma}

\begin{proof}
Consider the minimization problem
\begin{equation}
\min_{x\in\mathcal{X}} \|z-x\|_{2}^{2}.
\end{equation}
Because $\mathcal{X}$ is the Cartesian product of intervals and the
objective decomposes additively,
\begin{align}
\|z-x\|_{2}^{2}
  &= \sum_{j=1}^{d} (z_j-x_j)^{2}, \\
\arg\min_{x\in\mathcal{X}}\|z-x\|_{2}^{2}
  &= \bigl(\arg\min_{x_j\in[0,1]} (z_j-x_j)^{2}\bigr)_{j=1}^{d}. \label{eq:coord_sep}
\end{align}
For a single coordinate $z_j$ the scalar minimizer is
\begin{equation}
x_j^{\star}
  =
  \begin{cases}
    0, & z_j<0,\\
    z_j, & 0\le z_j \le 1,\\
    1, & z_j>1,
  \end{cases}
\end{equation}
which is exactly $(\clip(z))_j$.  Stacking these $x_j^{\star}$ yields
$y=\clip(z)$, and by \eqref{eq:coord_sep} this $y$ minimizes the full
problem.  Uniqueness follows from strict convexity of the squared norm,
hence $\clip(z)=\Pi_{\mathcal{X}}(z)$.
\end{proof}

Let us note that any projection operator is $1$-Lipschitz since the distance between any two points after projection is never greater than their distance before projection. We seek to bound the expected squared error $\mathcal{E} := \mathbb{E}[ \| \hat{T}_{\mathrm{MKL}}(X_0) - T^*(X_0) \|^2 ]$.

\begin{theorem}[Error Bound for L-Lipschitz Color Maps]
Let Assumptions 1 and 2 hold. The total error $\mathcal{E}$ is bounded as:
\begin{equation}
    \mathcal{E} \le 2\mathcal{E}_{clip} + 2\mathcal{E}_{lin},
\end{equation}
where the clipping error is $\mathcal{E}_{clip} := \mathbb{E} [ \| T_{MKL}(X_0) - \hat{T}_{MKL}(X_0) \|^2 ]$, and the linearity error, $\mathcal{E}_{lin}$, is bounded by:
\begin{equation}
    \mathcal{E}_{lin} \le 2 B^2 + 2 (\|A\|_{op} + L)^2 \cdot \mathrm{tr}(\Sigma_0).
\end{equation}
Here, $B = |\mu_1 - T^*(\mu_0)|$ is a bias term, $||A||_{op}$ is the spectral norm of the MKL matrix, i.e. its largest singular value, which depends only on source and target distribution covariances, and $\mathrm{tr}(\Sigma_0)$ is the trace of the source covariance matrix.
\end{theorem}
\begin{proof}
The decomposition $\mathcal{E} \le 2\mathcal{E}_{clip} + 2\mathcal{E}_{lin}$ follows from the inequality $\|u+v\|^2 \le 2\|u\|^2 + 2\|v\|^2$. We focus on bounding the linearity error term, $\mathcal{E}_{lin} = \mathbb{E} [ \| T_{MKL}(X_0) - T^*(X_0) \|^2 ]$. We introduce and subtract terms centered at the mean of the source distribution, $\mu_0$:
\begin{align*}
    &\| T_{MKL}(x) - T^*(x) \| \\ &= \| (\mu_1 + A(x-\mu_0)) - T^*(x) \| \\
    &= \| (\mu_1 - T^*(\mu_0)) + (A(x-\mu_0) - (T^*(x) - T^*(\mu_0))) \|.
\end{align*}
Applying the triangle inequality, $\|u+v\| \le \|u\| + \|v\|$, yields:
\begin{align*}
    &\| T_{MKL}(x) - T^*(x) \| \\ &\le \|\mu_1 - T^*(\mu_0)\| + \|A(x-\mu_0) - (T^*(x) - T^*(\mu_0))\|.
\end{align*}
Let $B = \|\mu_1 - T^*(\mu_0)\|$. For the second term, we again use the triangle inequality and the L-Lipschitz assumption on $T^*$:
\begin{align*}
    &\|A(x-\mu_0) - (T^*(x) - T^*(\mu_0))\| \\&\le  \|A(x-\mu_0)\| + \|T^*(x) - T^*(\mu_0)\| \\
    &\le \|A\|_{op}\|x-\mu_0\| + L\|x-\mu_0\| \\
    &= (\|A\|_{op} + L)\|x-\mu_0\|.
\end{align*}
Combining these yields the pointwise bound: $\| T_{MKL}(x) - T^*(x) \| \le B + (\|A\|_{op} + L)\|x-\mu_0\|$. To obtain the bound on $\mathcal{E}_{lin}$, we square this expression and take the expectation. Using $(u+v)^2 \le 2u^2 + 2v^2$:
\begin{align*}
    \mathcal{E}_{lin} &= \mathbb{E} [ \| T_{MKL}(X_0) - T^*(X_0) \|^2 ] \\
    &\le \mathbb{E} [ 2B^2 + 2(\|A\|_{op} + L)^2\|X_0-\mu_0\|^2 ].
\end{align*}
Since $\mathbb{E}[\|X_0-\mu_0\|^2] = \mathrm{tr}(\Sigma_0)$ by definition of covariance, the result follows.
\end{proof}

This theorem bounds the approximation error, linking it to the Lipschitz constant $L$ of the true optimal transport map and the tail probability $\mathcal{E}_{clip} \le d \cdot \mathbb{P}[T_{\mathrm{MKL}}(X_0) \notin \mathcal{X}]$. The latter bound holds due to the following Lemma \ref{lem:clip-tail}.

\begin{lemma}[Tail–probability bound for the clipping error]%
\label{lem:clip-tail-sup}
Let $\Pi_{\mathcal{X}}=\clip(\cdot)$ be the Euclidean projection onto
$\mathcal{X}=[0,1]^d$, defined as
\begin{align}
    (\clip(z))_j \;:=\; \min\bigl\{\,1,\;\max\{0,\;z_j\}\bigr\},
\end{align}
for $j=1,\dots,d,\;z\in\mathbb{R}^d$ and define
\begin{align}
\mathcal{E}_{\text{clip}}
\;:=\;
\mathbb{E}\!\bigl[\|Z-\Pi_{\mathcal{X}}Z\|^2\bigr],
\qquad Z\in\mathbb{R}^d.
\end{align}
Then
\begin{align}
\mathcal{E}_{\text{clip}}
  &\le
  d\,\mathbb{P}[Z\notin\mathcal{X}],
  \label{eq:clip-bound-sup}
\end{align}
and for our application with $Z=T_{\mathrm{MKL}}(X_0)$ and $d=3$,
\begin{align}
\mathcal{E}_{\text{clip}}
  &\le
    3\,\mathbb{P}\bigl[T_{\mathrm{MKL}}(X_0)\notin\mathcal{X}\bigr].
  \tag{\ref{eq:clip-bound-sup}$'$}
\end{align}
\end{lemma}

\begin{proof}
Because $\clip$ acts coordinate-wise,
\begin{equation*}
    |z_j-(\clip z)_j|\le 1 \quad\text{for each } j.
\end{equation*}
Hence
\begin{align*}
    \|z-\clip z\|^2
  =\sum_{j=1}^{d}\bigl(z_j-(\clip z)_j\bigr)^2
  \le d\,\mathbf{1}_{\{z\notin\mathcal{X}\}},
\end{align*}
where $\mathbf{1}_{\{\cdot\}}$ is the indicator function.
Taking expectations with $z=Z$ yields \eqref{eq:clip-bound-sup}.
\end{proof}

\section{3D Model Credits}

\begin{figure}[h]
    \centering
    \includegraphics[width=0.7\linewidth]{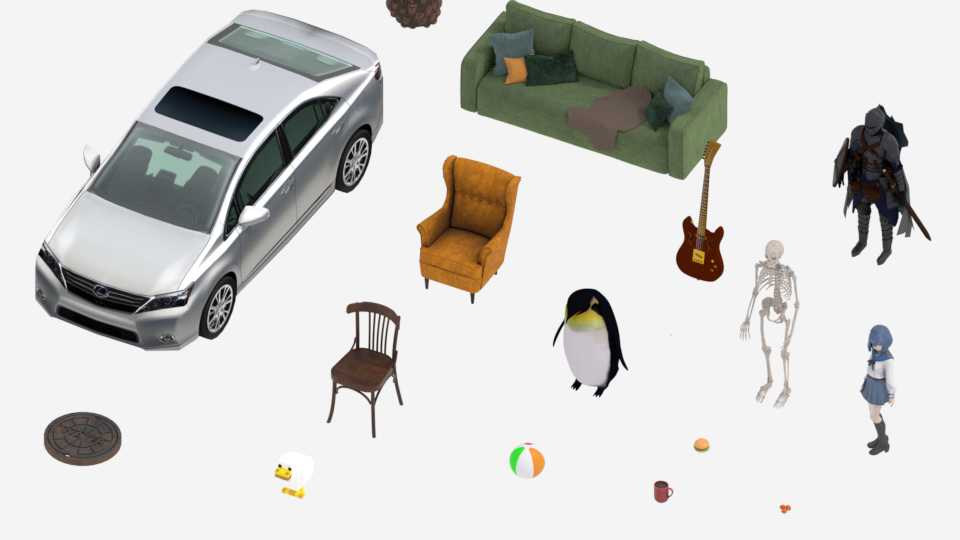}
    \caption{The 3D objects used in our experiments.}
    \label{fig:3d_objects_supp}
\end{figure} 

3D objects used in our experiments are collected from open source and converted to the Wavefront .obj format and single ``baked'' texture file. Some of them are shown in Fig.~\ref{fig:3d_objects_supp}. Models' authors and sources are listed below 

\begin{itemize}
    \item \textbf{AR Core pawn} \\
    \url{github.com/google-ar/arcore-android-sdk}

    \item \textbf{Chair} Author: River Yang, \\
    \url{https://www.blenderkit.com/asset-gallery-detail/703fe964-84bb-4468-84bd-f6aa3fb337b4/}
    
    \item \textbf{Chicken} 

    \item \textbf{Beach ball} Author: Mohammed Hossain, \\\url{https://www.blenderkit.com/asset-gallery-detail/cb647b58-d1c8-4a34-a6b8-263cd002bf30/}

    \item \textbf{Green sofa} Author: Joris LM, \\\url{https://www.blenderkit.com/asset-gallery-detail/573bf228-019b-4072-8086-bc45a6a2b2fa/}
 
    \item \textbf{IKEA armchair} Author: Branislav Kubečka,\\
    \url{https://www.blenderkit.com/asset-gallery-detail/82020230-9451-4683-bb6b-1e01c7c4fa36/}

    \item \textbf{Burger} Author: Raymond Gabriel, \\
    \url{https://www.blenderkit.com/asset-gallery-detail/5fe9d2e3-31ac-44d1-b392-cdb00ae1d490/}

    \item \textbf{Skeleton} Author: Aidan Sanderson,\\
    \url{https://www.blenderkit.com/asset-gallery-detail/130ec931-f4df-45b8-9a81-2660ccceb581/}

    \item \textbf{Knight's armor} Author: AnonmalyFound,\\ 
    \url{https://www.blenderkit.com/asset-gallery-detail/4d71b492-be49-4748-b159-6ec337aefa50/}

    \item \textbf{Car} Published by: WyattP, \\
    \url{https://open3dmodel.com/3d-models/lexus-free-3d-model-car_12127.html}

    \item \textbf{Tangerines} Author: Yahku le Roux,\\
    \url{https://www.blenderkit.com/asset-gallery-detail/fc4e728b-c25b-4059-bdc9-d211b4fadcf8/}

    \item \textbf{Manhole}  Author: Raunox, Poly Heaven \\
    \url{https://www.blenderkit.com/asset-gallery-detail/847ef18f-3f72-4893-b5c2-389a1898c21b/}

    \item \textbf{Giant fern} Author: Amanpreet Bajwa, \\
    \url{https://www.blenderkit.com/asset-gallery-detail/25d6a477-7a01-46b1-b5dd-d162fe2ab2fc/}

    \item \textbf{Penguin} Author: nabarun1011,\\
    \url{https://free3d.com/3d-model/emperor-penguin-601811.html},
    \url{https://sketchfab.com/3d-models/penguin-f65e799bf9534c66a12724a93bb72c39}

    \item \textbf{Coffee mug} Generated,\\ \url{https://www.meshy.ai/}

    \item \textbf{Anime girl} Generated,\\ \url{https://www.meshy.ai/}

    \item \textbf{Guitar} Author: nabarun1011,\\\url{https://sketchfab.com/3d-models/electric-guitar-d51983607d404791acb42f460259f23a}

    \item \textbf{Beagle dog} \\
    \url{https://www.cadnav.com/3d-models/model-55200.html}

    \item \textbf{Deer} Author: TheCaitasaurus, \\
    \url{https://open3dmodel.com/3d-models/little-deer-animal_352076.html}

    \item \textbf{Pigeon} Published by: Zoe\_A, \\
    \url{https://open3dmodel.com/3d-models/3d-model-rock-pigeon_119874.html}
\end{itemize}

\begin{figure}[h]
    \centering
    \includegraphics[width=1\linewidth]{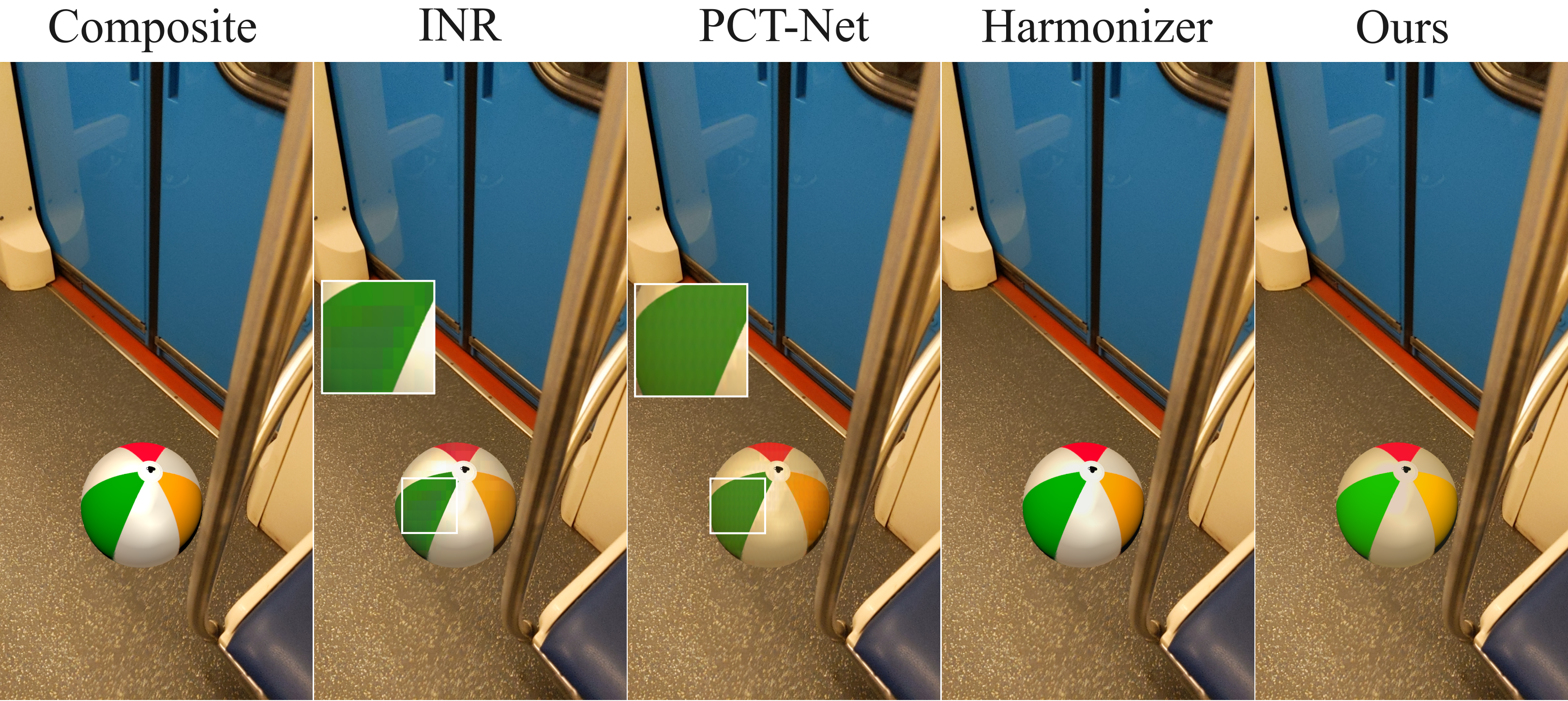} 
    \includegraphics[width=1\linewidth]{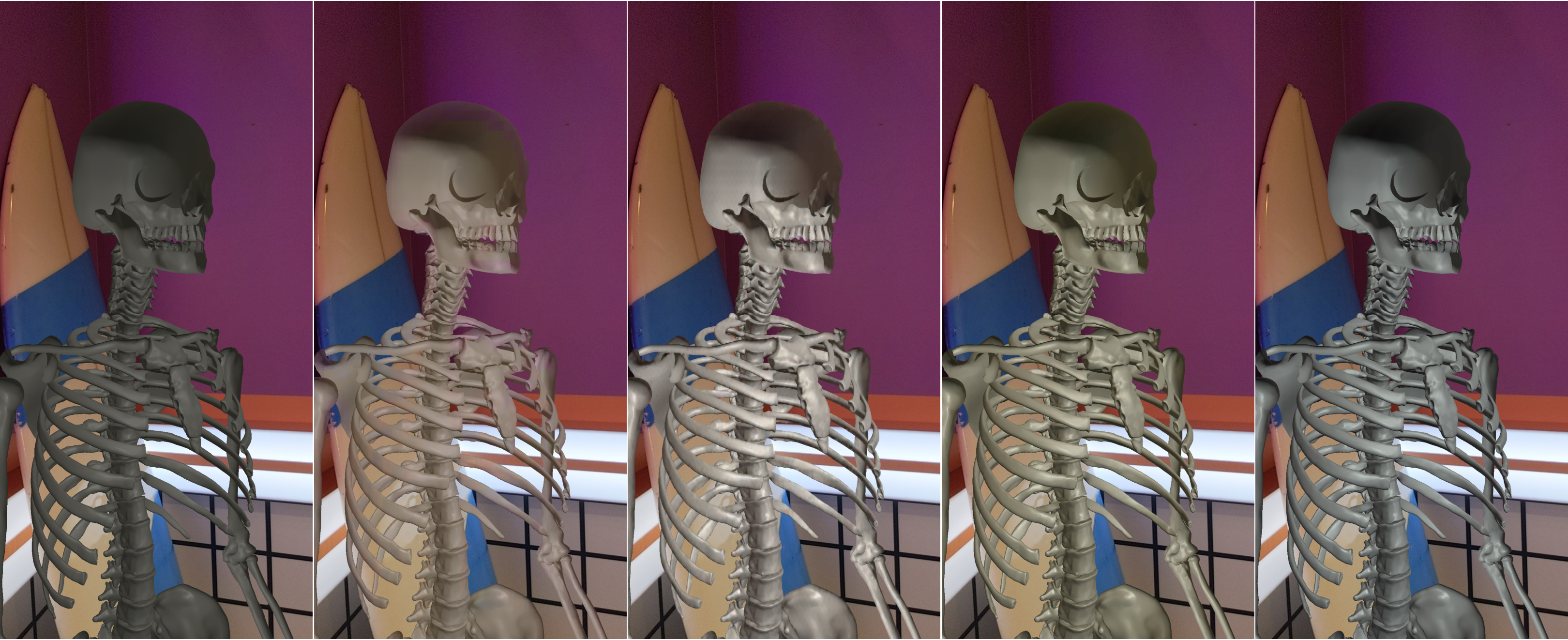} 
    \includegraphics[width=1\linewidth]{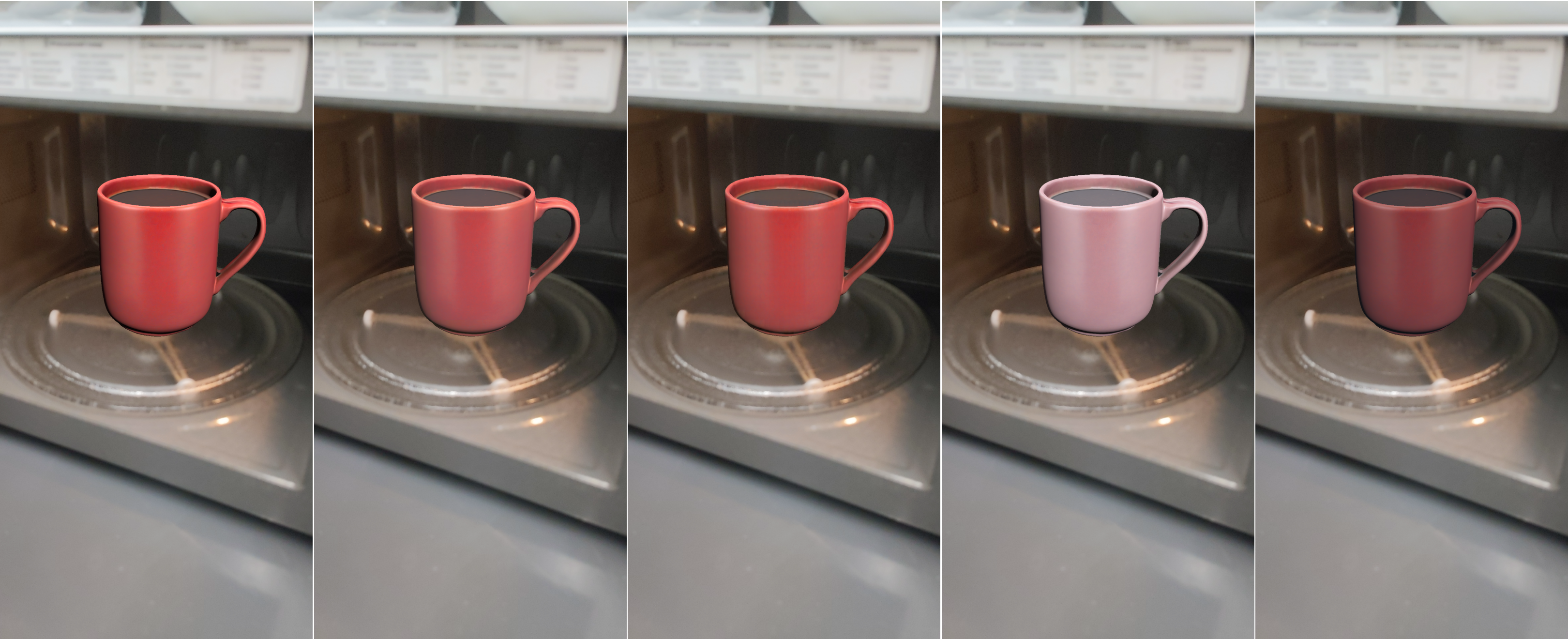} 
    \includegraphics[width=1\linewidth]{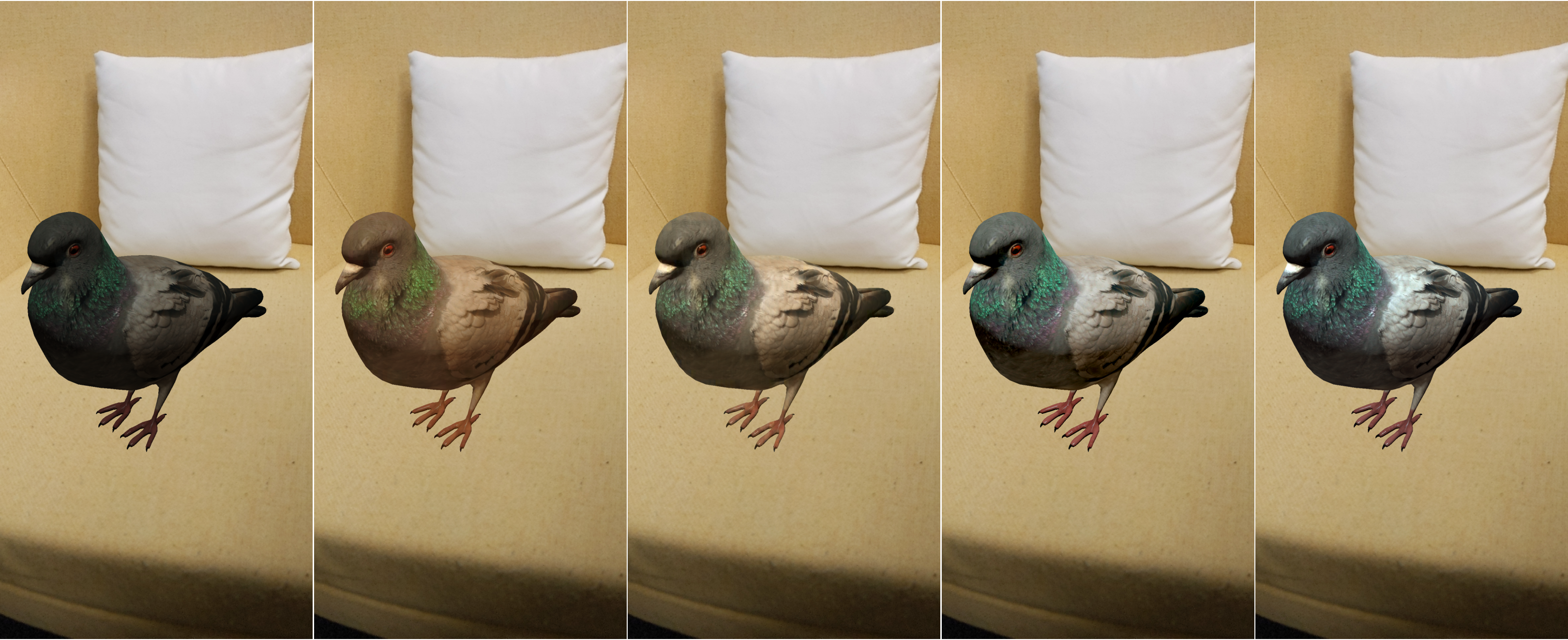}  
    \includegraphics[width=1\linewidth]{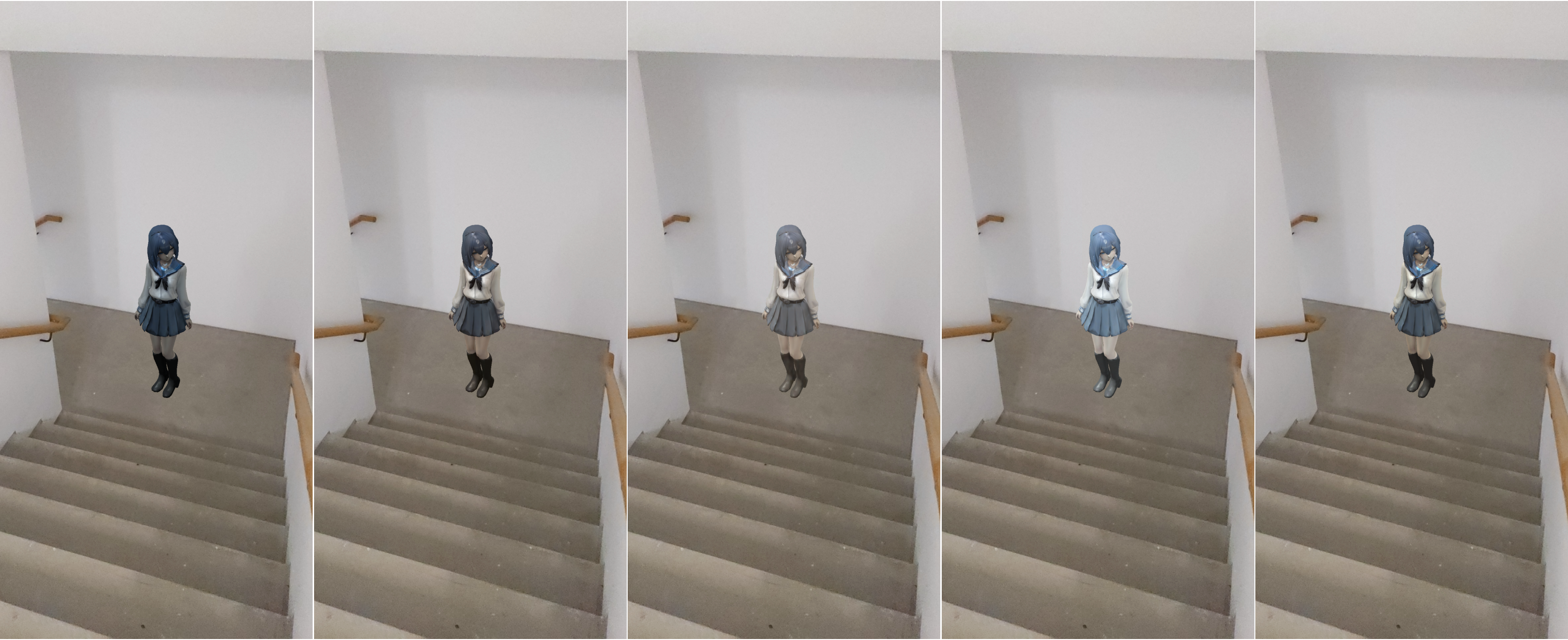}  
    \includegraphics[width=1\linewidth]{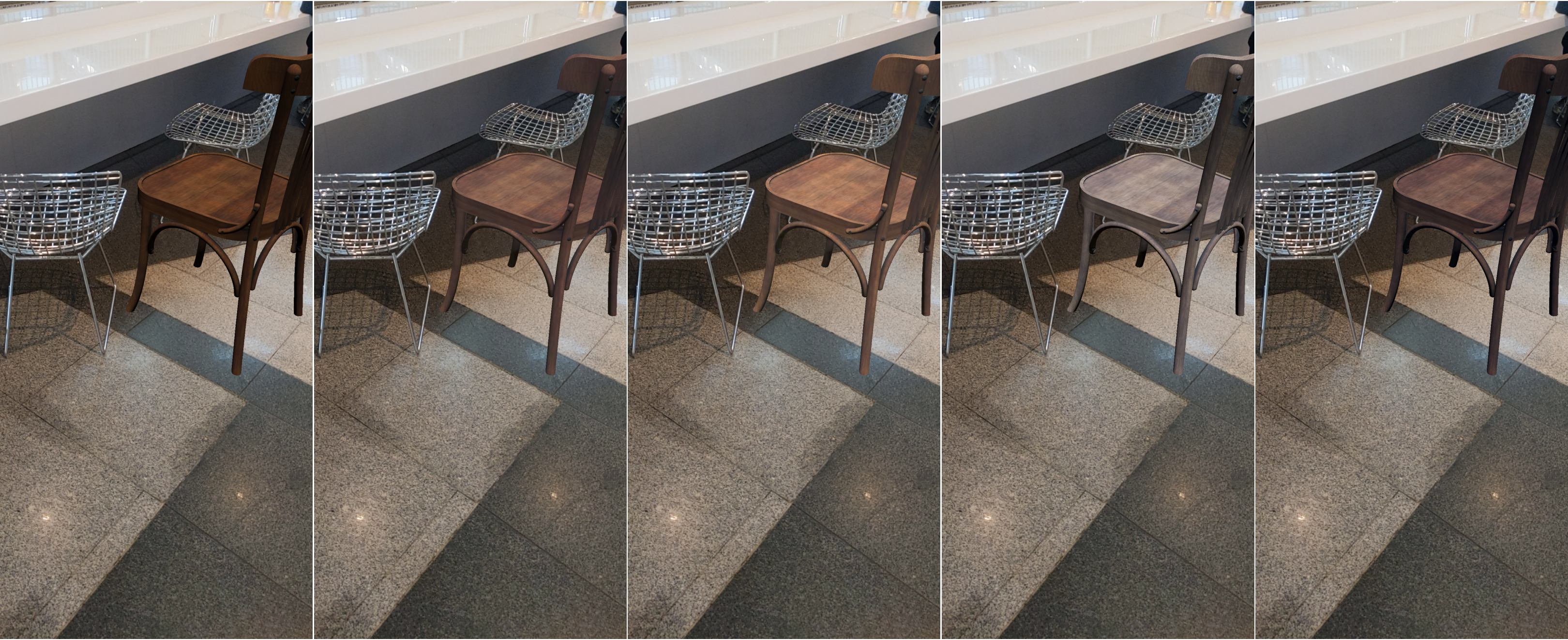}
    \caption{Qualitative comparison with baselines.}
    \label{fig:qualitative_comparsion_supp}
\end{figure}

\begin{figure*}[t]
    \centering
    \includegraphics[width=0.75\linewidth]{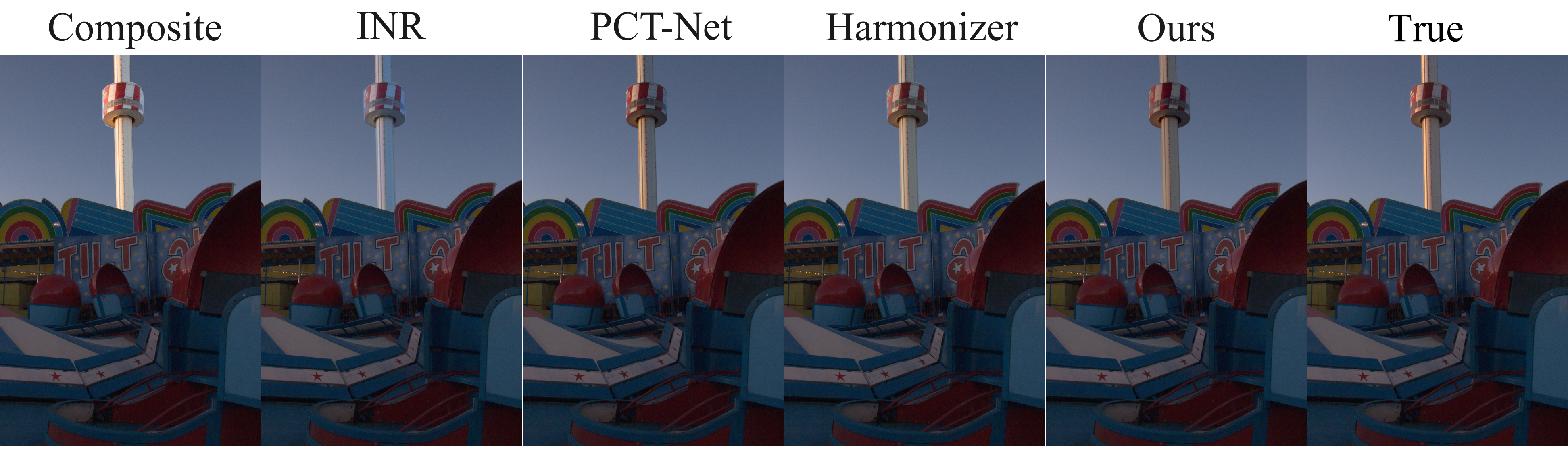}
    \includegraphics[width=0.75\linewidth]{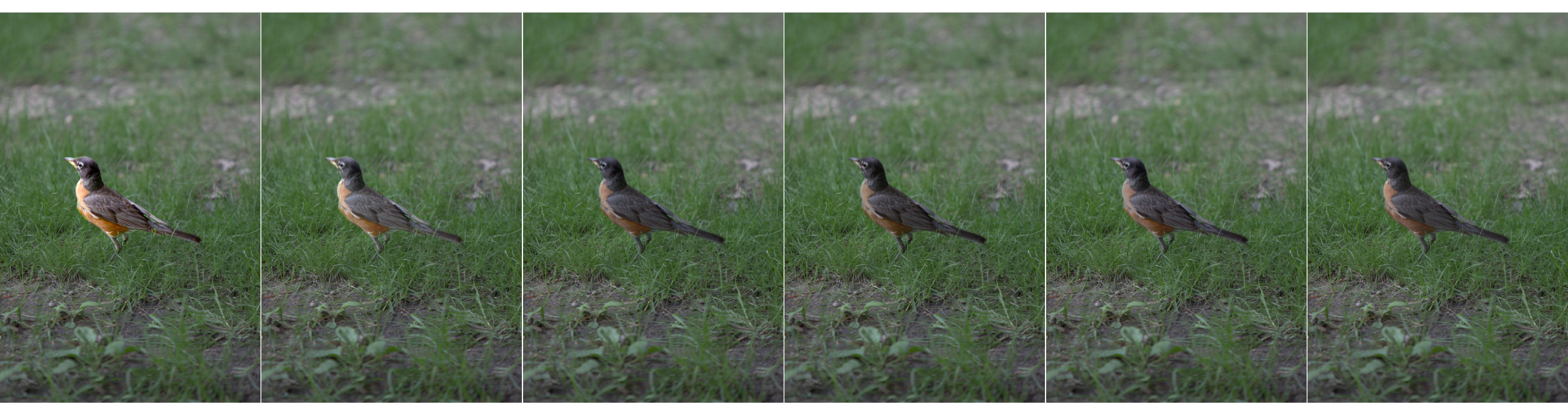}
    \includegraphics[width=0.75\linewidth]{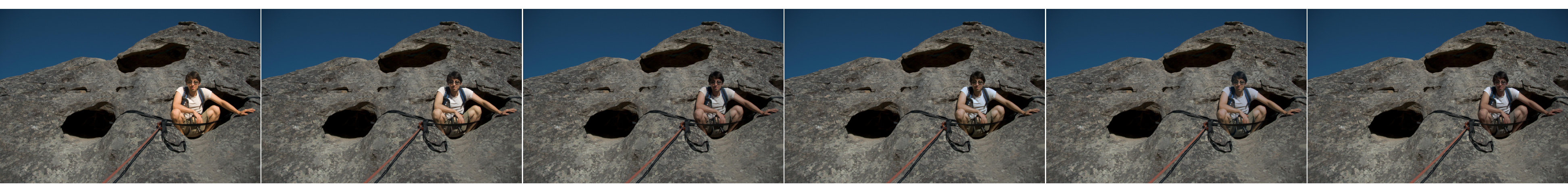}
    \includegraphics[width=0.75\linewidth]{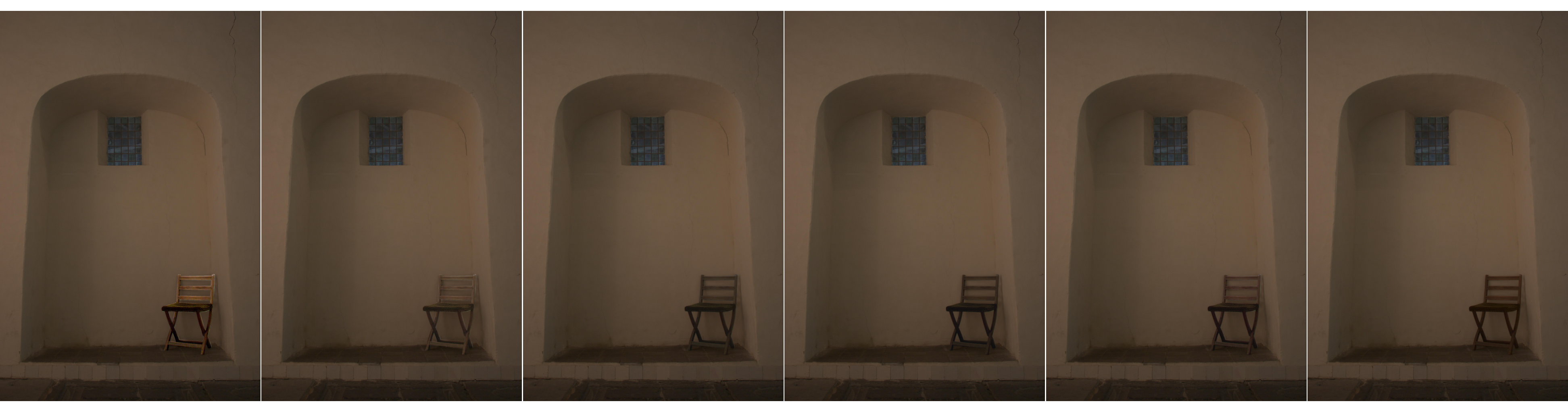}
    \includegraphics[width=0.75\linewidth]{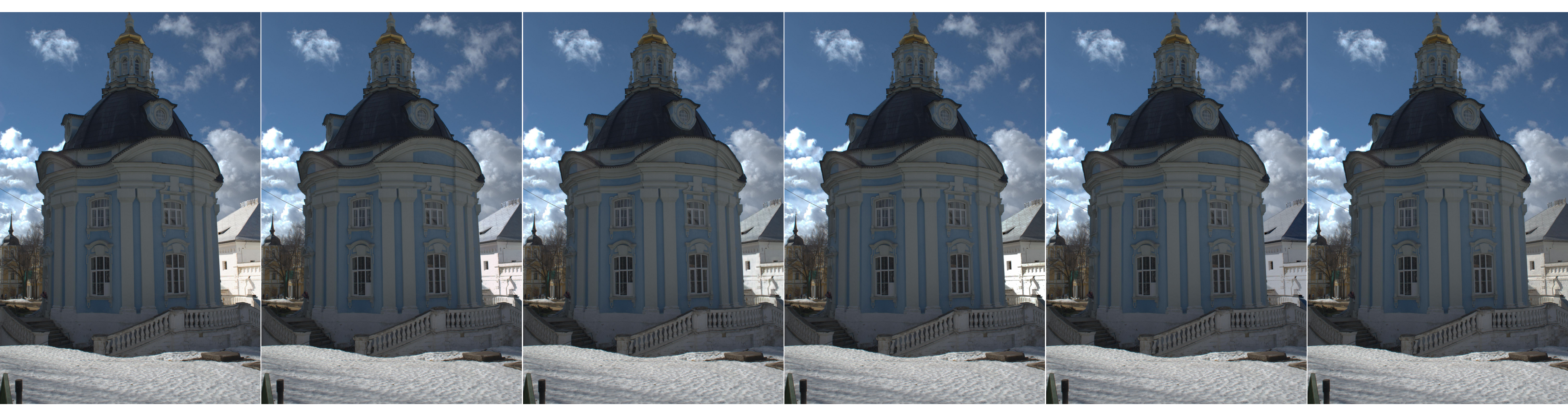}
    \includegraphics[width=0.75\linewidth]{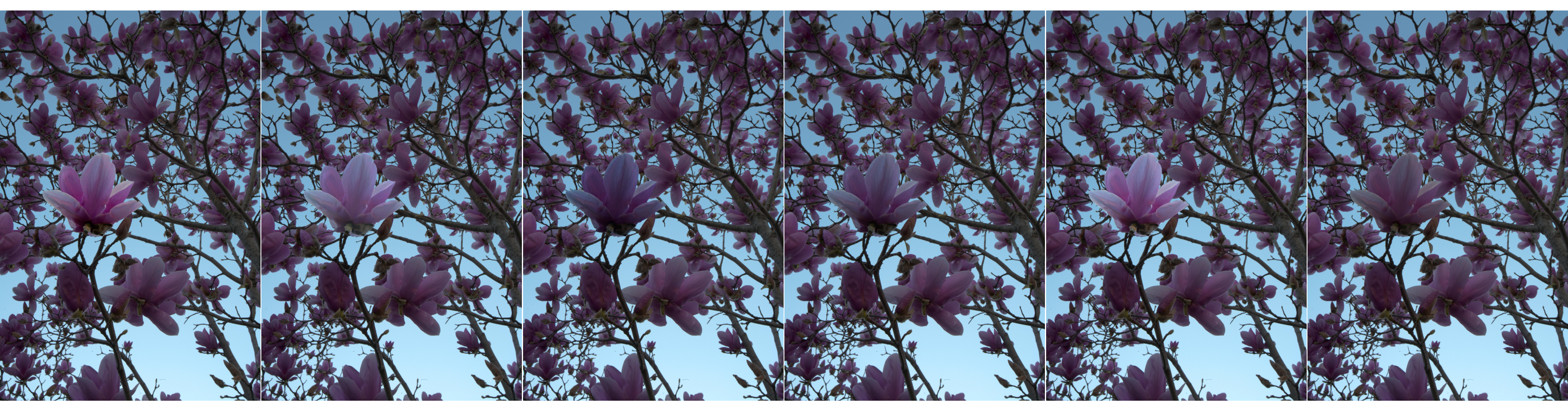}
    \includegraphics[width=0.75\linewidth]{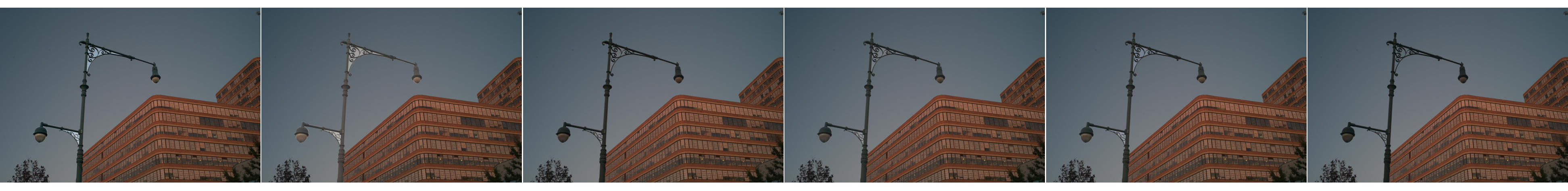}
    \caption{Qualitative comparison with baselines and true images on HAdobe5k test set.}
    \label{fig:qualitative_comparsion_iharmony}
\end{figure*}

\end{document}